\documentclass[sigconf]{acmart}
\AtBeginDocument{%
  }

\copyrightyear{2023} 
\acmYear{2023} 
\setcopyright{acmlicensed}\acmConference[MM '23]{Proceedings of the 31st ACM International Conference on Multimedia}{October 29-November 3, 2023}{Ottawa, ON, Canada}
\acmBooktitle{Proceedings of the 31st ACM International Conference on Multimedia (MM '23), October 29-November 3, 2023, Ottawa, ON, Canada}
\acmPrice{15.00}
\acmDOI{10.1145/3581783.3612072}
\acmISBN{979-8-4007-0108-5/23/10}

\usepackage{algorithm}
\usepackage{algpseudocode}
\usepackage{xspace} 
\usepackage{amsthm}

\usepackage{amssymb}
\newtheorem{myDef}{\textbf{Definition}} 
\newtheorem{myTheo}{\textbf{Theorem}}
\usepackage{color}
\usepackage{multirow}
\usepackage{threeparttable}  
\usepackage{colortbl}
\usepackage[switch]{lineno}
\usepackage{wrapfig}

\begin{document} 
\title{Model Inversion Attack via Dynamic Memory Learning}


\author{Gege Qi}
\email{qigege.qgg@alibaba-inc.com}
\affiliation{%
  \institution{Alibaba Group}
  \city{Hang Zhou}
  \country{China}
}

\author{YueFeng Chen}
\email{yuefeng@alibaba-inc.com}
\affiliation{%
  \institution{Alibaba Group}
  \city{Hang Zhou}
  \country{China}}

\author{Xiaofeng Mao}
\email{mxf164419@alibaba-inc.com}
\affiliation{%
  \institution{Alibaba Group}
  \city{Hang Zhou}
  \country{China}
}

\author{Binyuan Hui}
\email{binyuan.hby@alibaba-inc.com}
\affiliation{%
 \institution{Alibaba Group}
 \city{Beijing}
 \country{China}}

\author{Xiaodan Li}
\email{fiona.lxd@alibaba-inc.com}
\affiliation{%
  \institution{Alibaba Group}
  \city{Hang Zhou}
  \country{China}}

\author{Rong Zhang}
\email{stone.zhangr@alibaba-inc.com}
\affiliation{%
  \institution{Alibaba Group}
  \city{Hang Zhou}
  \country{China}}

\author{Hui Xue}
\email{hui.xueh@alibaba-inc.com}
\affiliation{%
  \institution{Alibaba Group}
  \city{Hang Zhou}
  \country{China}}

\renewcommand{\shortauthors}{Gege Qi et al.}

\newcommand{\name}{\textsc{Dmmia}\xspace}
\def\etal{\emph{et al.}}
\def\eg{\emph{e.g.}}
\def\ie{\emph{i.e.}}
\renewcommand{\proofname}{\textbf{Proof}}
\begin{abstract}
Model Inversion (MI) attacks aim to recover the private training data from the target model, which has raised security concerns about the deployment of DNNs in practice. 
Recent advances in generative adversarial models have rendered them particularly effective in MI attacks, primarily due to their ability to generate high-fidelity and perceptually realistic images that closely resemble the target data. 
In this work, we propose a novel Dynamic Memory Model Inversion Attack (\name) to leverage historically learned knowledge, which interacts with samples (during the training) to induce diverse generations.
\name constructs two types of prototypes to inject the information about historically learned knowledge: Intra-class Multicentric Representation (IMR) representing target-related concepts by multiple learnable prototypes, and Inter-class Discriminative Representation (IDR) characterizing the memorized samples as learned prototypes to capture more privacy-related information.
As a result, our \name has a more informative representation, which brings more diverse and discriminative generated results.
Experiments on multiple benchmarks show that \name performs better than state-of-the-art MI attack methods. 
\begin{figure}
  \centering
  \includegraphics[width=1.0\linewidth]{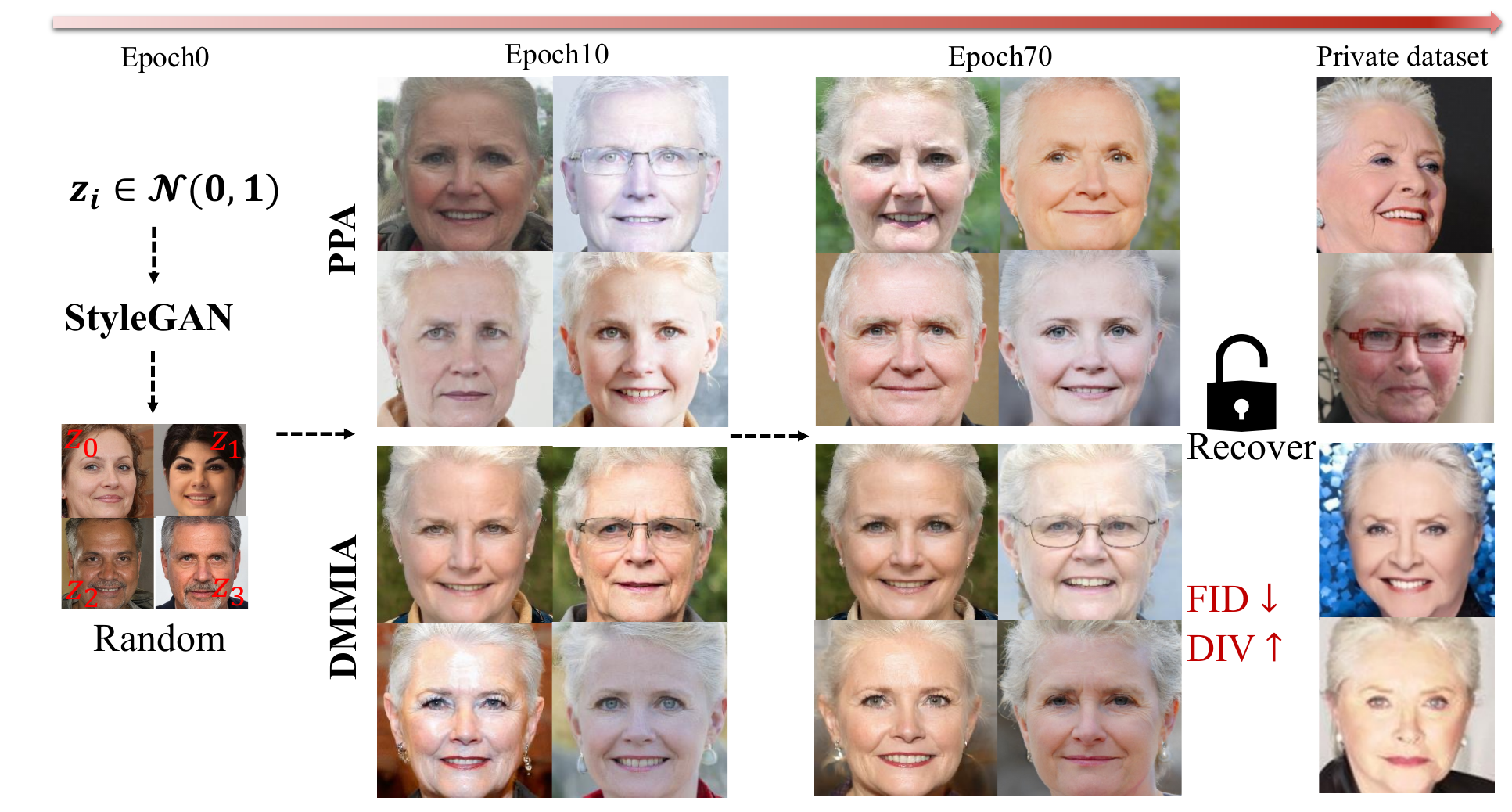}
  \vspace{-0.5cm}
  \caption{Comparison between \name and the baseline model inversion attack. \name has higher sample diversity while achieving more accurate privacy data reconstructions. In this example, our method can preserve more characteristics (\eg, glasses) of the target class during optimization.}
  \vspace{-0.6cm}
  \label{fig:intro}
\end{figure}
\end{abstract}



\keywords{Model Inversion Attack, Intra-class Multicentric Representation, Inter-class Discriminative Representation\let\thefootnote\relax\footnotetext{This research is supported in part by the National Key Research and Development Program of China under Grant No.2020AAA0140000.}}



\maketitle

\section{Introduction}
Deep neural networks (DNNs) have achieved great success in a wide range of applications, including facial recognition \cite{parkhi2015deep,schroff2015facenet}, personalized medicine \cite{academy2015stratified} and product recommendation \cite{wu2017session}. However, the fact that privacy-sensitive applications of DNNs increasingly utilize sensitive and proprietary information raised great concerns about privacy. Recently, \cite{fredrikson2015model} proposed a model inversion (MI) attack aiming to reconstruct sensitive features of private training data by leveraging their correlation with the model output. The study highlights the vulnerability of linear regression models used in personalized medicine to such attacks, leading to the disclosure of confidential genomic data of individual patients.

While MI attacks have traditionally been limited to discrete signals, recent advances have extended the attack to high-dimensional and continuous signals such as image data using Generative Adversarial Networks (GANs) \cite{zhang2020secret,hidano2017model,chen2021knowledge}.
However, the effectiveness of GAN-based MI attacks is limited by an infamous problem called "catastrophic forgetting" \cite{lesort2019generative}. In MI attacks, as shown in Figure~\ref{fig:intro}, this tendency is expressed as that the generated samples from early updating epochs contain more characteristics (\eg, glasses), but some are dropped during training progress. This property leads to poor attack performance and limited diversity in generation results.


One common approach to mitigate catastrophic forgetting in GANs is to employ a progressive training strategy, which involves gradually adding new characteristics to the generated samples while preserving previously learned ones, such as gradient penalties \cite{thanh2018catastrophic}. Another approach is to use a memory module that can store and recall previously learned information \cite{chenshen2018memory}.
Our approach leverages memory replay mechanisms to maintain previously learned features while adapting to new characteristics, resulting in more effective MI attacks and increased diversity in generated samples.
By addressing the issue of forgetting in GAN-based MI attacks, our work represents a significant step forward in the development of efficient MI attacks without modifying the architecture of GANs.

Towards this end, we propose a novel MI attack leveraging a Dynamic Memory Mechanism, which effectively reuses the memory of the inversion data by prototype learning, named as \name.
Specifically, a prototype is defined as ``a representative embedding for a group of semantically similar instances.''
\name composes two types of prototypes, which are designed to memorize intra- and inter-classes information, respectively. In detail: (1) \textbf{Intra-class Multicentric Representation (IMR)}: IMR uses multiple centers stored by prototypes to present the concept of the target class, which increases the sample diversity from the perspective of intra-class relationships. (2) \textbf{Inter-class Discriminative Representation (IDR):} Given a specific class, it stores the historical knowledge of previously synthesized images and enforces the embedding feature of a sample to be more similar to its corresponding prototypes compared to other prototypes. 

Along with progressively-updated prototypes, \name can well preserve the diversity of the target compared with the previous method, as shown in Figure~\ref{fig:intro}.
However, one might wonder why this memory mechanism is helpful in increasing sample diversity. We perform a theoretical study on the sample diversity of \name.
In fact, the floor level of the geodesic distance (the shortest path between the vertices) between any two generated images is larger than that without proposed prototypes (see Section~\ref{sec:diversity}).
It confirms that \name can preserve more diversities of the target distribution.

We perform extensive experiments with both sensitive face and natural datasets, diverse model architectures, and different levels of image prior. Experimental results validate that \name improves the attack success rate and can synthesize high fidelity and diverse images. Towards attacks with and without high-quality image prior, we outperform state-of-the-art methods by up to 6.26\% and 20.8\% accuracy rates respectively.
Our contributions are as follows: 
\begin{itemize}
\item We propose a novel dynamic memory model inversion attack (\name), which applies an Intra-class Multicentric Representation (IMR) term and an Inter-class Discriminative Representation (IDM) term to model the representation of target classes by memory induction.
\item  We give both theoretical and empirical insights to validate \name's effectiveness.
\item Our proposed \name achieves new state-of-the-art attack performance on multiple benchmarks.
\end{itemize}

\begin{figure*}
  \centering
  \includegraphics[width=0.85\linewidth]{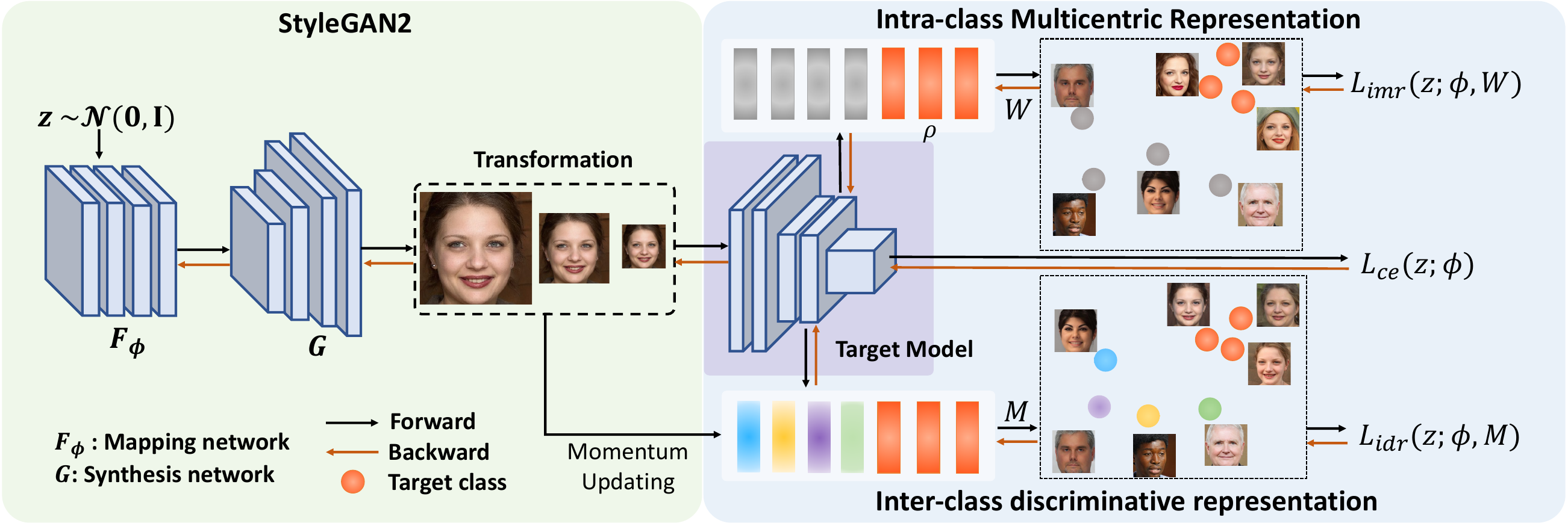}
  \vspace{-0.4cm}
  \caption{Overview of Dynamic Memory Model Inversion Attack. Given a target model, the sampled vector z from the Gaussian distribution is first transformed to a disentangled latent code with semantic meanings by Mapping Network. Then, the Synthesis network modifies this latent code to an image. With proposed IMR and IDR, the memory representations are modeled for updating Mapping Network and learnable W by back-propagating the proposed losses.}
  \vspace{-0.3cm}
  \label{fig:pipline}
\end{figure*}

\section{Related Work}

\paragraph{Catastrophic forgetting.} 
Methods for overcoming catastrophic forgetting can be categorized as: (1) Continual learning-based methods, like EWC~\cite{kirkpatrick2017overcoming}, aim to minimize weight changes in critical areas for previous tasks by estimating the diagonal empirical Fisher information matrix. Based on EWC, \cite{seff2017continual} has explored label-conditioned image generation to enhance image realism. However, it faces challenges in generating high-quality images. Lifelong GAN \cite{zhai2019lifelong} utilizes knowledge distillation to transfer knowledge from prior networks to the news. Piggyback GAN~\cite{zhai2020piggyback} is a parameter-efficient approach by weight reusing when extending a trained network to new tasks. These methods aim to enable models to adapt to new tasks without knowledge forgetting. However, MI attacks focus on exploiting vulnerabilities in the model's behavior to extract sensitive information rather than preserving knowledge across tasks. (2) Incremental learning is a relatively work to study the problem of catastrophic forgetting. DCIGAN~\cite{guan2019dcigan} uses a GAN generator to store the information of past data and continuously update GAN parameters with new data. For MI attacks, learning about the target data is an incremental process, making it challenging to train auxiliary GANs to record historical information effectively.
(3) Our work aligns with the concept of memory replay, where images from a model trained on previous tasks are combined with current training images for updates.

\paragraph{Model inversion attacks.} 
Our work aims to recover private training data or sensitive attributes from access to a trained model. One of the earliest MI attacks specific to a linear regression model was proposed on genomic privacy \cite{fredrikson2014privacy}. For threat models with high-dimensional inputs, \cite{fredrikson2015model} first used gradient descent to solve the underlying attack optimization problem. To improve the attack performance in image space, Generative Adversarial Networks (GANs) were introduced for synthesizing high-quality samples \cite{chen2020improved,zhang2020secret,khosravy2022model}. Further, due to the great success of StyleGAN2 \cite{karras2020analyzing} in generating high fidelity images, \cite{wang2021variational} built upon this generator for higher sample realism. 
Though effective, these methods heavily depend on the distribution shift between the auxiliary and private data. For example, by pretraining on auxiliary datasets in terms of blurred or partially blocked images, \cite{zhang2020secret} improved the identification accuracy of reconstructed images. \cite{struppek2022plug} introduced a robust and flexible attack, which only built the target upon the independent image priors. Existing MI attacks have overlooked the catastrophic forgetting issue. In addressing this problem, we find that reusing historical information can be advantageous in improving the capture of the target data distribution while preserving diversity.


Moreover, several papers considered the black-box setting without access to a model’s gradients \cite{wu2016methodology,yang2019adversarial,aivodji2019gamin,MIRROR}. In this paper, we take a new perspective to address the white-box MI attacks. We aim to capture more identity-representative information from historically synthesized images. We strive for dynamic memory learning via defined prototypes, which helps to prevent forgetting during training, resulting in high threats to target models.

\section{Problem Formulation}

In white-box MI attacks, attackers have full access to the target model trained on a private dataset $\mathcal{D}=\{(x_i,y_i),i\in(1,\cdots, N)\}$, drawn from joint distribution $P_{\mathbb{X},\mathbb{Y}}$, where $x_i \in \mathbb{X}$ and $y_i \in \mathbb{Y}=\{1, \cdots, K\}$ denote the input image and the corresponding class label, respectively. For target $K$ class classifer $f_c: \mathbb{X}  \rightarrow \mathbb{Y}  \in \mathbb{R}^K$, the goal of is to synthesis $\mathcal{D}^S\{\hat{x}_i|y_i = y_t\}$ to approximate the training data $\mathcal{D}^X\{x_i| y_i = y_t\}$, where $y_t$ is the specific target label.

However, when $x$ is high-dimensional data, directly optimizing the objective may directly result in meaningless features for generated results. To address this issue, generative models are introduced to generate high fidelity and semantic images \cite{wang2021variational,chen2021knowledge,zhao2021exploiting}. 
Thus the attack process switches to producing private data that matches the target model using a generative model.
As shown in Figure~\ref{fig:pipline}, we use a pre-trained StyleGAN \cite{karras2019style} with auxiliary knowledge as the generative models, which  the attack's knowledge of the target model's classification intent. Instead of feeding the input latent code $z$ directly to the beginning of the network, the \emph{mapping network} $F_\phi(z)$ first converts it to an intermediate latent code, which is then fed into a \emph{synthesis network} $G$ to synthesize images. 
Thus, the optimization of the $x$ becomes the optimization of the network parameter $\phi$ in the mapping network, and the MI attack is defined as follows:
\begin{align}\label{eq:MIAF}
\min_{\phi} \mathcal{L}(f_c( \hat{x} ),  y_t), \hat{x} = G\left(F_\phi(z)\right), z \sim \mathcal{N}(0, I)
\end{align}
To avoid catastrophically forgetting previously learned knowledge, in the following section, we propose a solution for optimizing the above objective function by constructing progressive prototypes to reuse historical knowledge.

\section{Methodology}
We propose a Dynamic Memory Model Inversion Attack (\name) to improve the generated results by exploring representation prototypes from two clues: 1) Intra-class Multicentric Representation (IMR). 2) Inter-class Discriminative Representation (IDR).
The overview of our method is depicted in Figure~\ref{fig:pipline}.

\subsection{Intra-class multicentric representation} 

To take full advantage of learned knowledge, we design an intra-class multicentric representation (IMR) prototype module. IMR represents the target class with multiple concepts through learnable prototypes, where each prototype represents a specific concept. 
we find that the generated samples tend to overfit by emphasizing the specific features for the target model, thereby losing the diversity within the class. For learning multiple concepts for each class by multicentric prototypes, we divide the learnable prototypes into positive and negative components. The positive part captures the salient characteristics associated with the target, while the negative part represents undesirable features to be avoided. By aligning the embedding feature of the generated sample with the prototypes, the IMR module effectively increases the sample diversity.

Formally, we define a set of trainable parameters $W\in\mathbb{R}^{N_w \times N_d}$ as learnable prototypes, where the positive prototype set $W^{p}=\{w_i\in\mathbb{R}^{1 \times N_d}, i\in[1:\rho]\}$ records multiple concepts for the target class and negative prototype set $W^n=\{w_i\in\mathbb{R}^{1 \times N_d}, i\in[\rho +1:N_w]\}$ are prototypes used to distinguish from the target class.
Let $f_e(\hat{x})$ denote a tensor image feature of dimension $N_d$, obtained by passing a generated image $\hat{x}$ through the feature extractor of the target model. We interpret each positive prototype in $W^p$ as a concept in the feature space of the target class. Then, the probability of a concept $f_e(\hat{x})$ in feature space belonging to the target class can be defined as:
\begin{align}
p_{imr}(\hat{x}) = \frac{\sum\nolimits_{i=1}^{\rho} \exp \left(f_e(\hat{x})^{T}w_i\right)}{\sum\nolimits_{j=1}^{N_w} \exp \left({f_e(\hat{x})^{T} w_j}\right)}
\end{align}
IMR aims to find the mapping network parameters $\phi$ and prototypes $W$ that maximize the log-likelihood function of the generated samples. Finally, IMR learns target class concepts by minimizing:
\begin{gather}
\mathcal{L}_{\rm imr}(z;\phi,W) = -\log p_{imr}(\hat{x}), \nonumber \\ 
where \quad \hat{x} = G(F_{\phi}(z))
\end{gather}
Consequently, through the learning of multiple concepts of the target class via prototype learning, IMR increases the sample diversity from the perspective of intra-class relations.

\begin{algorithm}
  \caption{\name}
  \label{alg:dmmia}
  \begin{algorithmic}[1]
    \State \textbf{Input:} Synthesis net $G$; Mapping net $F_{\phi}$; Target net $f$; Momentum coefficient $r$; Training epoch $T$
    \State \textbf{Initialize}: IMR set $W$; Empty IDR set $M$
    \Procedure{Update}{$W,M,F_{\phi}$}
      \While{Training step $i<T$}
        \State $z_i \gets \mathcal{N}(0, I)$\Comment{Sample latent vector}
        \State $\hat{x}_i \gets G(F_{\phi}(z_i))$
        \State $ \phi_{i+1} \gets\nabla_{\phi_i}\mathcal{L}_{\name}(\hat{x}_i,y_t,\phi_{i}, W_i, M_i)$
        \State $W_{i+1} \gets  \nabla_{W_i}\mathcal{L}_{imr}(\hat{x}_i,y_t,W_{i})$\Comment{Fix $\phi_{i+1}$}
        \State $M_{i+1} \gets rM_{i} + (1-r)M^{'}_{i}$\Comment{update memory bank}
      \EndWhile\label{euclidendwhile}
    \EndProcedure
  \end{algorithmic}
\end{algorithm}
\setlength{\abovedisplayskip}{3pt}
\subsection{Inter-class discriminative representation}
Besides exploring the target representation from intra-relation via IMR, we design an inter-class discriminative representation (IDR) module. As aforementioned, the previously generated unique features are generally discarded during optimization. IDR defines a set of non-parameter prototypes by maintaining a memory bank, where the embedding features $f_e(\hat{x})$ of generated images are stored for prototype metric learning. 
Specifically, historical features of input images are injected into the IDR prototype set $M_c$ according to their predicted class $c=\arg\max p(\hat{x})$,  where $p(\hat{x})$ represents the prediction probability of target model.
We denote the IDR prototypes as $M\in\mathbb{R}^{K \times N_d}$, where $M_{i}\in\mathbb{R}^{N_d}$ records the feature in memory for the i-th class.  
Then, IDR repels different categories by measuring the similarity between the current embedding feature $f_e(\hat{x})$ and each class of discriminative prototypes. The probability of a generated sample $\hat{x}$ belonging to the target class is denoted as:
\setlength{\abovedisplayskip}{3pt}
\begin{align}
p_{idr}(\hat{x}) = \frac{\exp \left({ f_e(\hat{x})^T M_{y_t}} \right)}{\sum\nolimits_{i=1}^{K} \exp \left({f_e(\hat{x})^TM_{i} }\right)}
\end{align}
\setlength{\abovedisplayskip}{3pt}
Finally, $f_e(\hat{x})$ is aligned with corresponding prototypes in embedding space by:
\setlength{\abovedisplayskip}{3pt}
\begin{gather}
\mathcal{L}_{\rm idr}(z;\phi,M) = -\log p_{idr}(\hat{x}), \nonumber \\ 
where \quad \hat{x} = G(F_{\phi}(z))
\end{gather}
\setlength{\abovedisplayskip}{3pt}
\paragraph{Momentum Updating.}
With the updating of $F_{\phi}$ of IDR during training, the memory bank is updated by the current image feature $f_e(\hat{x})$.
To update the memorized prototypes in a more stable way, we propose to use a momentum update as follows:
\begin{equation}
\label{eq:total_dm_}
M_{c} = rM_{c} + (1-r)M^{'}_c 
\end{equation}
where $c=\arg\max p(\hat{x})$ is the predicted class and $r\in[0,1]$ is a momentum coefficient. $M^{'}_{c}$ indicates the mean of current prototype features of the c-th class. With progressively-updated memorized prototypes, $\mathcal{L}_{\rm idr}$ represents the privacy data from historical knowledge. Consequently, IDR encourages the generated images in the same class to have more distinguishable characteristics, \ie, revealing more privacy-sensitive information.

\subsection{Overview of the training}
Following a previous attack method \cite{zhang2020secret}, cross-entropy loss \cite{de2005tutorial} $\mathcal{L}_{\rm ce}$ is used to enforce the inverted images to be correctly classified as the target label $y_{t}$. With the memory learning loss defined, the objective loss function is:
\setlength{\abovedisplayskip}{3pt}
\begin{align}\label{eq:total_dm}
\mathcal{L}_{\rm \name} &= \mathcal{L}_{\rm ce} + \lambda_1 \cdot \mathcal{L}_{\rm imr} + \lambda_2 \cdot \mathcal{L}_{\rm idr}
\end{align} 
\setlength{\abovedisplayskip}{3pt}
where $\lambda_1$ and $\lambda_2$ are scalars balancing the influence of two prototypes. We remark that the $\phi$ and $W$ are optimized alternately, and the details of \name are presented in Algorithm~\ref{alg:dmmia}. 
Firstly, we sample a predetermined number of latent vectors $z$, and then only the 200 latent vectors with the highest prediction scores are used for optimization.

\section{Analysis on the Diversity of \name}
\label{sec:diversity}
In this subsection, we analyze how \name affects the intra-class diversity, which is an important metric of a good MI attack. We analyze the relationship between the probability distribution and the memory projected probability distribution from the perspective of KL-divergence. Without loss of generality, we assume $N_w = K$, the IMR and IDR become $W \in  \mathbb{R}^{K \times N_d}$ and $M\in \mathbb{R}^{K \times N_d}$. We define $p(y|\hat{x}) = softmax f_c(\hat{x})$ as the probability distribution on target model under the $\mathcal{L}_{ce}$ supervision. 
The memory projected probability distribution is then $q(y|\hat{x}) = softmax \left((W+M)^T f_e(\hat{x}) \right)$.

\begin{myDef}
Let $f:\mathbb{X} \rightarrow \mathbb{Y}$ be a smooth mapping. When $W \in  \mathbb{R}^{K \times N_d }$ and $ M\in  \mathbb{R}^{K \times N_d }$, $(W+M)^T f_e(\hat{x})$ can be represented as a linear transformation of $f_e(\hat{x})$. There exists an $\eta$ such that $f_e(\hat{x}+\eta) = (W+M)^T f_e(\hat{x})$.
\end{myDef}
\begin{myDef}
Let $\mathbb{X}$ be a Riemannian manifold with the Fisher information matrix $\mathbf{G}_{\hat{x}}$ as its metric tensor. $s=\left[p_{1}(\hat{x}),\cdots,p_{K}(\hat{x})\right]^{T}$ is the output of the softmax layer in function $f$, where $\mathbf{G}_{s}$ is the Fisher information matrix associated with $s$. 
\end{myDef}

According to the defination, $\hat{x} + \eta$ represents the generated optimal sample corresponding to $\mathcal{L}_{\rm imr}$ and $\mathcal{L}_{\rm idr}$. 
We can use $p(y|\hat{x} + \eta)$ to approximate the distribution of $q(y|\hat{x})$. 
Thus, our goal is converted to observe the KL-divergence between $p(y|\hat{x})$ and $p(y|\hat{x} + \eta)$. 
During the optimization, $\hat{x}$ and $\hat{x} + \eta$ are progressively aligned to the specified target class, enabling consistent decreasing KL-divergence between the corresponding probability distribution from target models.
Assuming that a sufficiently small $\eta$ is obtained by Algorithm~\ref{alg:dmmia}, an intriguing property of \name is defined as follows.
\begin{myTheo}
Let $\hat{x} = G\left(F_\phi(z)\right)$ be the synthesised image. Minimizing the KL-divergence between the distribution $p(y|\hat{x})$ and $p(y|\hat{x}+\eta)$ is encourages by optimizing the $\mathcal{L}_{\rm \name}$, which is equivalent to: $\arg\min \sum_{i = 1}^{K}\frac{1}{p_i}$, $\operatorname{s.t.} \sum_{i=1}^{K}p_{i} = 1$. 
Then, we have $p_1 = p_2 = ... = p_{K} = \frac{1}{K}$, where $p_i$ is the prediction probability of i-th class.
\end{myTheo}
\begin{proof}
\renewcommand{\qedsymbol}{}
Let $y$ be the random variable ranging from $y_1$ to $y_k$. 
The KL-divergence between the distribution $p(y|\hat{x})$ and $p(y|\hat{x}+\eta)$ can be expanded via the second-order Taylor expansion:
\setlength{\abovedisplayskip}{3pt}
\begin{gather}
D_{\rm KL}(p(y|\hat{x}),p(y|\hat{x} + \eta )) = \mathbb{E}_{y}[\log\frac{p(y|\hat{x})}{p(y|\hat{x} + \eta)}] \approx \frac{1}{2} \eta^{T}\mathbf{G}_{\hat{x}}\eta \nonumber\\
\mathbf{G}_{\hat{x}} = \mathbb{E}_{x\sim \mathbb{X},y\sim p(y|\hat{x})}\left[g_x(\hat{x},y)g_x(\hat{x},y)^T\right]
\end{gather}
\setlength{\abovedisplayskip}{3pt}
where $\mathbf{G}_{\hat{x}} $ is the fisher information matrix of $\hat{x}$, and $g_x(\hat{x},y)$ is the gradient to input $\hat{x}$ w.r.t on the loss for label $y$. It is difficult to compute the high dimension matrix $\mathbf{G}_{\hat{x}}$. We follow \cite{shen2019defending} to formulate $\mathbf{G}_{\hat{x}}$ as a new matrix $\mathbf{G}_{s}$ through $\mathbf{G}_{\hat{x}} =J^T\mathbf{G}_{s}J$. $\mathbf{G}_{s}$ is the fisher information matrix of the output of the softmax layer $s=[p_1(\hat{x}),\cdots,p_K(\hat{x})]^T$, which has been defined in Definition 2. The term $J$ is the Jacobian matrix of $f_{c}$ and can be computed by $J = \frac{\partial s_i}{\partial \hat{x}}$. To this end, $\mathbf{G}_s$ is a $K \times K$ positive definite matrix and formulated as:
\begin{align}
\mathbf{G}_{s} = \mathbb{E}_{s\sim \mathbb{Y},y\sim p(y|s)}[g_s(s,y)g_s(s,y)^T] 
\end{align}
where $g_s(s,y)$ is the gradient to $s$ w.r.t on the loss for label $y$. Then, Equation 10 becomes:
\setlength{\abovedisplayskip}{3pt}
\begin{align}
\label{app_Gf}
\eta^T\mathbf{G}_{\hat{x}}\eta = \eta^TJ^T\mathbf{G}_{s}J\eta 
\end{align}
\setlength{\abovedisplayskip}{3pt}
During the optimization of GANs, $\hat{x}$ and $\hat{x} + \eta$ are aligned to the target class while gradually minimizing the KL-divergence between them. This minimization decreases the variances of $\eta$, while seeking the smallest eigenvalue of $\mathbf{G}_{\hat{x}}$, \ie, minimizing the eigenvalues of $\mathbf{G}_{s}$. The trace of metric equals the summation of all eigenvalues, which are all positive. Hence, minimizing eigenvalues is equivalent to finding the smallest trace. The trace of $\mathbf{G}_{s}$ can be computed as:
\begin{equation}
\begin{split}
tr(\mathbf{G}_{s}) &= tr(\mathbb{E}_{s\sim \mathbb{Y},y\sim p(y|s)}[g_s(s,y)g_s(s,y)^T]) \\
&= tr(\mathbb{E}_{s}[(\nabla_{s}\operatorname{log}p(y|s)) (\nabla_{s}\operatorname{log}p(y|s))^T]) \\
&= \int_y p(y|s)[tr(\nabla_{s}\operatorname{log}p(y|s))^T(\nabla_{s}\operatorname{log}p(y|s))] \\
&= \sum_{i=1}^{K}p_{i} \sum_{j = 1}^{K}(\nabla_{p_{j}}\operatorname{log}p_{i})^{2}  = \sum_{i = 1}^{K}\frac{1}{p_i} 
\end{split}
\end{equation}
Then, minimizing $tr(\mathbf{G}_{s})$ is equivalent to finding the optimal solution of $\arg\operatorname{min} \sum_{i = 1}^{K}\frac{1}{p_i}, \operatorname{s.t.} \sum_{i=1}^{K}p_{i} = 1$.
\end{proof} 
Different from general-purpose, which only matches the target class-specific feature, the memory-guided prototype learning enforces $F_{\phi}$ optimizing towards uniform distribution $\mathcal{U}(1,k)$. It may be different from intuition, but this intriguing property is no harm due to the main term $\mathcal{L}_{\rm ce}$ in the loss Equation~\ref{eq:total_dm}. Together with the above analysis, we can obtain the following theorem.

\begin{myTheo}
The geodesic distance in the probability space is always no larger than the corresponding distance in the data space:
\begin{align}
\label{app_lemma}
 \mathcal{D}(\hat{x}_i,\hat{x}_j) \geq \mathcal{D}(softmax(f_c(\hat{x}_i)),softmax(f_c(\hat{x}_j))) 
\end{align}
where $\mathcal{D}$ is the geodesic distance. 
\end{myTheo}

\begin{proof}
\renewcommand{\qedsymbol}{}
Note that for most neural networks, the dimensionality of $s$ is much less than that of $\hat{x}$.
That is to say, $\eta$ is mapped by $J$ from high dimensional data space to low dimensional probability space. Thus, $f_{c}$ is a surjective mapping and $\mathbf{G}_{\hat{x}}$ is a degenerative metric tensor. Following \cite{zhao2019adversarial}, the geodesic distance in the probability space is always no larger than the corresponding distance in the data space. Assuming that $\hat{x}_i$ and $\hat{x}_j$ are two generated images from the same target class $y_t$. Then, we have the inequality as Equation~\ref{app_lemma}.
\vspace{-0.55cm}
\end{proof}
Without memory prototypes, the target model predicted probability is updated toward a pure point $p(y|\hat{x})=[0,...,p_{y_t},...,0]$, among which the valid values are only those with an index $y_t$ and all the others 0. However, Theorem 1 shows that the memory prototypes cause the $F_{\phi}$ to update toward a uniform distribution while conforming to private target data. Thus \name further leads to a mixed probability distribution together $p(y|\hat{x})$ with $\mathcal{U}(1,K)$.
This objective probability distribution is more likely to attain high variance than pure $p(y|\hat{x})$. Thus, under the \name scenario, the floor level of the geodesic distance between any two generated images in probability space is larger than that of baseline. At a higher level of data space, according to Theorem 2, the distance of generated two images is increased consistently. It indicates that these prototypes can increase the sample similarity. In a broader sense, the theorem confirms the validity of our approach.

\begin{table*}[!t]
\renewcommand\arraystretch{1.3}
    \centering
    \scalebox{1.0}{
    \begin{tabular}{cl ccccc cccc}
        \toprule
        \textbf{\textsc{Dataset}} & \textbf{\textsc{Method}} & \textbf{\textsc{$\uparrow$ Acc@1}} & \textbf{\textsc{$\uparrow$ Acc@5}}& \textbf{\textsc{$\downarrow l_{eval}^{2}$}} & \textbf{\textsc{$\downarrow l_{face}^{cos}$}} & \textbf{\textsc{$\downarrow$FID}} & \textbf{\textsc{$\uparrow$Precision}} & \textbf{\textsc{$\uparrow$Recall}} & \textbf{\textsc{$\uparrow$Density}} & \textbf{\textsc{$\uparrow$Coverage}} \\
        \midrule
         \multirow{5}{*}{\textbf{\rotatebox{90}{CelebA}}} &
     GMI \cite{zhang2020secret}& 9.76 & 15.22 & 225.41 & 1.7800 & 151.06 & 0.0829 & 0.0283& 0.0366 & 0.0379\\
 	 ~ & KED \cite{chen2021knowledge}& 14.71 & 28.69 & 197.03 & 1.4023 & 132.34 & 0.0097 & 0.0072 & 0.0056 & 0.0011 \\  
 	~ & VMI \cite{wang2021variational}& 59.74 & 74.32 & 152.24 & 0.8871 & 79.31 & 0.0210 & 0.0198 & 0.0238 & 0.0320 \\
    ~ & MIRROR \cite{MIRROR} & 77.90 & 83.40 & 138.62 & 0.9991 & 78.12 & 0.1523 & 0.0246 & 0.6080& 0.2873 \\
 	  ~ & PPA \cite{struppek2022plug} & 87.76 & 96.50 & 129.02 & 0.7218 & 59.73 & 0.2856 & 0.0513 & 0.7304 & 0.4123 \\
 	 ~ & \name & \textbf{94.02} & \textbf{99.33} & \textbf{123.54} & \textbf{0.7167} & \textbf{58.29} & \textbf{0.2956} & \textbf{0.0581} & \textbf{0.7383} & \textbf{0.4270}  \\
         \midrule
         \multirow{5}{*}{\textbf{\rotatebox{90}{FaceScrub}}}  &
 	 GMI \cite{zhang2020secret}& 12.30 & 21.96 & 163.43 & 1.4880 & 199.61 & 0.0082 & 0.0176 & 0.0174 & 0.0094 \\
 	 ~ & KED \cite{chen2021knowledge}& 16.05 & 24.74 & 152.40 & 1.2075 & 157.63 & 0.0194 & 0.0018 & 0.0058 & 0.0029 \\  
 	~ & VMI \cite{wang2021variational}& 51.77 & 78.94 & 142.78 & 0.9071 & 98.09 & 0.0343 & 0.0307 & 0.0329 & 0.0455 \\
  ~ & MIRROR \cite{MIRROR} & 82.31 & 91.74 & 139.68 & 0.9511 & 63.82 & 0.0210 & 0.0238 & 0.0605& 0.0378 \\
 	  ~ & PPA \cite{struppek2022plug} & 85.76 & 97.56 &114.99 & 0.7155 & 49.93 & 0.1427 & 0.0462 & 0.1777 & 0.1870   \\
    
 	 ~ & \name & \textbf{93.54} & \textbf{99.28} & \textbf{110.73} & \textbf{0.6392} & \textbf{47.48} & \textbf{0.1513} & \textbf{0.0571} & \textbf{0.1894} & \textbf{0.1906} \\
        \bottomrule
    \end{tabular}}
     \caption{Comparison with state-of-the-art methods against ResNet-18 trained on CelebA and FaceScrub respectively. $\uparrow$ and $\downarrow$ symbolize that higher and lower scores give better attack performance.}
  \vspace{-0.8cm}
  \label{tbl:benchnmark}
    \label{tab:ablation}
\end{table*}
\section{Experiments}
\label{sec:exp}


\subsection{Training Setups}

\paragraph{\textbf{Dataset.}}
We evaluate \name on two face image datasets, including CelebFaces Attributes \cite{liu2015deep} (CelebA) and FaceScrub Dataset \cite{ng2014data}. CelebA contains 10,177 identities with coarse alignment. The train-test-splits setting is followed with \cite{struppek2022plug}.
FaceScrub is consisted of face images of male and female celebrities, with about 200 images per person. Moreover, we validate the effectiveness of \name on Stanford Dogs dataset \cite{deng2009imagenet} and hand-written image dataset MNIST \cite{lecun2010mnist}. Due to cost concerns, we used 1/10 classes in ablation studies. More details about the dataset can be referred to Appendix B.

\paragraph{\textbf{Models.}}
We implement several popular target networks with different depths and architectures, including ResNet-18 \cite{he2016deep} nets on CelebA and FaceScrub. In the extended evaluation section, we attack the ResNeSt-101 \cite{zhang2020resnest}, ResNet-152 \cite{he2016deep} and DenseNet-169 \cite{huang2017densely} networks separately on target datasets. For generators, we consider StyleGAN2 with Flickr-Faces-HQ (FFHQ) \cite{karras2019style}, MetFaces \cite{karras2020training} and Animal Faces-HQ Dogs (AFHQ) prior for inversion. FFHQ offers higher quality human face images, while MetFaces is a face dataset extracted from the Metropolitan Museum of Art collection, which has a large distribution shift from real face images. Besides, AFHQ \cite{choi2020stargan} is used as a prior for attacking Stanford Dogs models. 

\paragraph{\textbf{Metrics.}}
We evaluate the MI attack performance from the perspective of target attack accuracy, sample diversity, and sample realism. Acc@1 and Acc@5 are defined as the top-1 and top-5 accuracy given a classifier for evaluation, respectively. We choose the Inception-v3 \cite{szegedy2016rethinking} trained on specific target datasets to test the accuracy over 50 generated samples for each target class. It is worth noting that this evaluation model is different from the target model. Higher accuracy indicates that the synthesized images reveal more private information about the target label. Furthermore, we measure the image quality by computing Frechet Inception Distance (FID) \cite{heusel2017gans} between reconstructions and private target images. For a comprehensive comparison, we compute the shortest feature distance of each generated image to target training samples. $l^{2}_{eval}$ and $l^{cos}_{face}$ represent the feature distance on the evaluation model Inception-v3 \cite{szegedy2016rethinking} and a pre-trained FaceNet \cite{schroff2015facenet}, respectively. Moreover, we report Precision-Recall \cite{kynkaanniemi2019improved} and Coverage-Density \cite{naeem2020reliable} metrics to quantify the intra-class diversity of synthesized samples explicitly. 

\paragraph{\textbf{Attack implementation.}}
The image prior assumption for pre-training StyleGAN2 models allows the target data and the auxiliary data are meant to be disjoint. The size of sampled latent vectors $z$ is 5000 for CelebA models and 2000 for FaceScrub and Stanford Dogs models. To train the mapping network $F_\phi(z)$ in StyleGAN2 and IMR prototypes $W$, we use the Adam optimizer with the learning rate 0.005 and batch size of 16, $\beta_{1}=0.1$ and $\beta_{2}=0.1$ \cite{kingma2015adam}. The training epochs on FaceScrub, Stanford Dogs, and CelebA experiments are set as 50, 50, and 70, respectively. Besides, the transformation operation follows \cite{struppek2022plug}.
We set the prototype number $N_w=500$ with $\rho=250$, $\lambda_1=0.3$ and $\lambda_2=0.7$. $r$ is set as 0.7. We perform all our experiments on 8 NVIDIA RTX 3090 GPUs, which are based on CUDA 11.4 and Python 3.8.10. For more details of the implementations, please refer to Appendix B. 

\subsection{Experiment results}
To thoroughly evaluate our \name, we conduct experiments on two privacy-preserving human-face datasets and two non-privacy datasets, accompanied by a detailed ablation study.
Additionally, we have analyzed the computational complexity of \name, and find that it adds only a little extra computational cost during training compared to the baseline, with no extra cost in inference attacks. 

\begin{figure}
  \centering
  \includegraphics[width=0.9\linewidth]{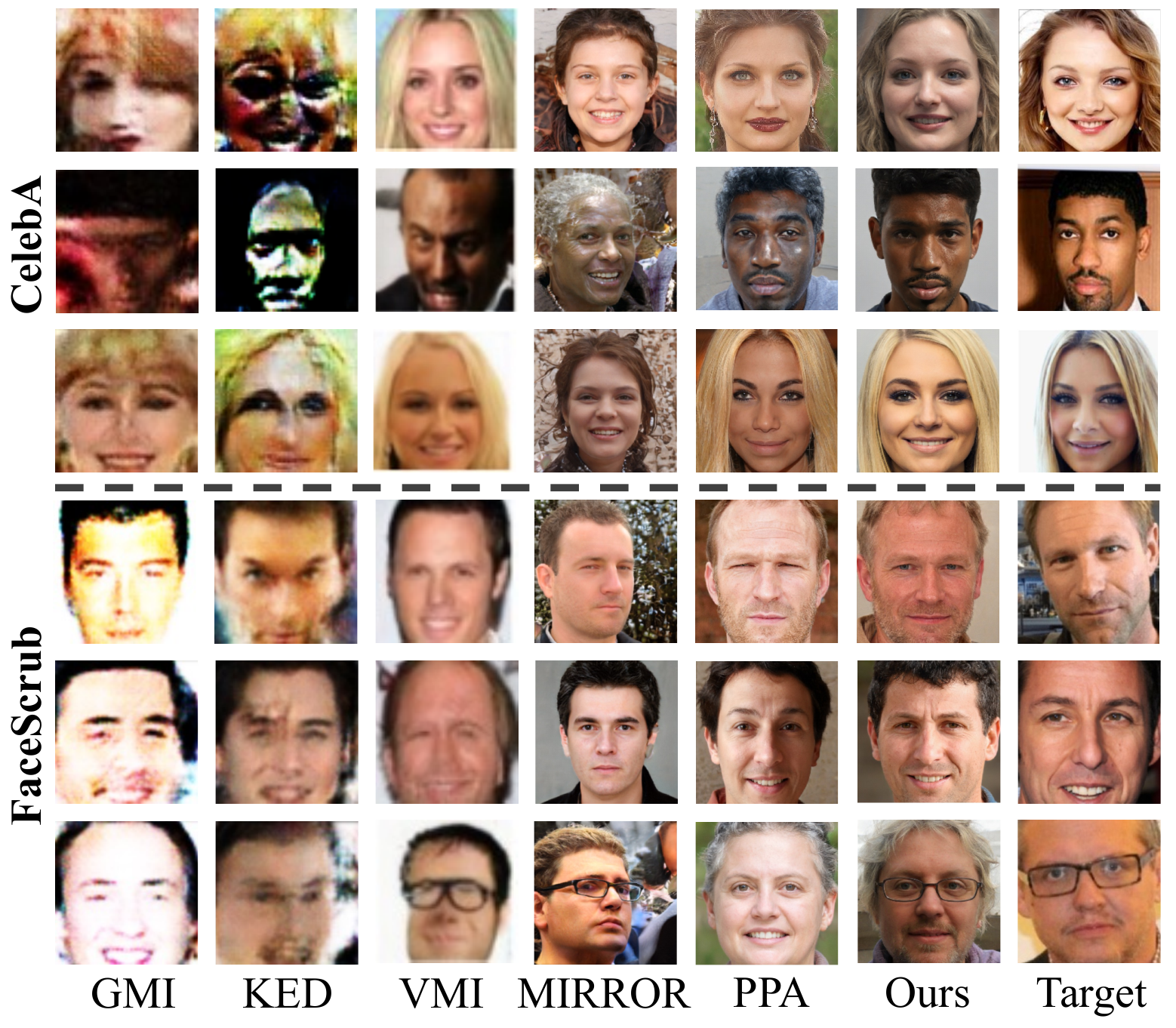}
  \vspace{-0.5cm}
  \caption{Visual comparison of attack results against the ResNet-18 trained on CelebA and FaceScrub.}
  \vspace{-0.6cm}
  \label{fig:vis_benchmark}
\end{figure}

\noindent \paragraph{\textbf{Benchmark results on privacy data.}}
We conduct a comprehensive evaluation of the proposed \name compared with several MI attacks, as shown in Table~\ref{tbl:benchnmark}. The generators for all methods are pre-trained on the FFHQ dataset and try to reveal the private data in CelebA and FaceScrub datasets. 
We find that \name can achieve state-of-the-art attack accuracy.
For CelebA datasets, the \textsc{Acc@1} of \name is higher than PPA by 6.26\%, while having a smaller distance between generated images and private data. The improved four diversity metrics indicate that our method can learn more characteristics of facial images per class. Additionally, \name outperforms all MI attacks on recovering the FaceScrub dataset, achieving the highest attack accuracy 93.54\% with better sample realism and diversity. We provide generated samples in 
Figure~\ref{fig:vis_benchmark}. Notice that the synthesized images by GMI and KED are at low resolution $64\times 64$ pixels. Besides, GMI and KED can not create realistic samples in our setting. VMI is not effective in generating face images with high fidelity. 
The faces inverted by MIRROR and PPA are more human-recognizable than the attacks mentioned above.
However, the images generated using \name show maximum similarity with the target privacy data.
\begin{figure}[t]
  \centering
  \includegraphics[width=1\linewidth]{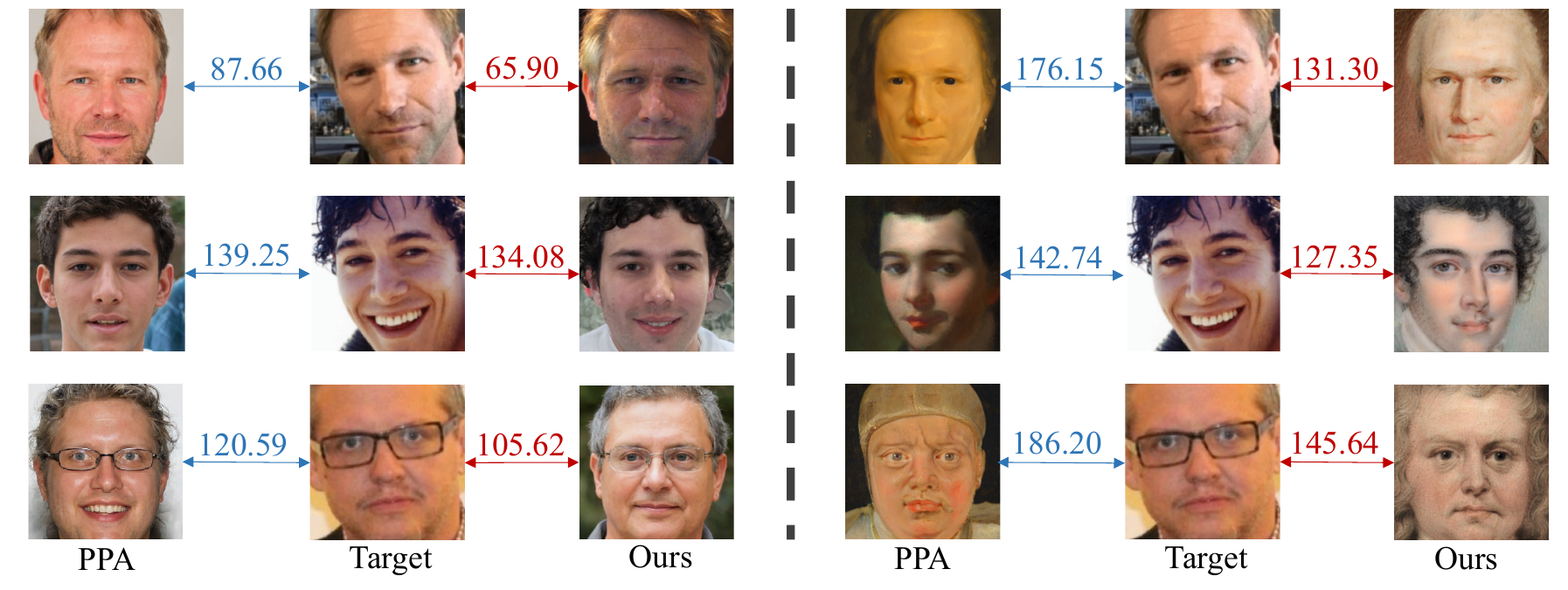}
  \vspace{-0.6cm}
  \caption{Visual results against the ResNeSt-101 with FFHQ (left) and MetFaces (right) prior. We show the quantitative similarity between the generated images and targets using $\downarrow l_{eval}^{2}$ metric. The images produced via \name have lower distance values, particularly when the prior distribution differed significantly from the target image distribution.}
  \vspace{-0.3cm}
  \label{fig:vis_2}
\end{figure}

\noindent \paragraph{\textbf{Extended evaluation.}}
We compare \name with the PPA with FFHQ and MetFaces image priors, as shown in Table~\ref{tbl:extend}.
See Appendix C for more experiments on attacking ResNeSt-101 and DenseNet-169.
We find that \name achieves higher attack accuracy across multiple target models.
For example, when attacking ResNet-152 trained on CelebA with FFHQ prior, the \textsc{Acc@1} of \name is higher than PPA by 2.38\%. Meanwhile, the sample realism and diversity of \name outperform PPA in all cases. 
With the loosened image prior, \name yields significantly better attack performance than baseline.
For attacking FaceScrub with Metfaces prior, the \textsc{Acc@1} of \name outperforms PPA by a large margin of 22.5\%. The distance between generated and target images on the evaluation model decreases by 5.94 on \textsc{$l_{eval}^{2}$}.
Figure~\ref{fig:vis_2} and the qualitative evaluation both support our findings, showing that the generated samples by \name had lower $\downarrow l_{eval}^{2}$ than the target data, particularly when attacking without knowledge of sensitive attributes using StyleGAN2 trained on MetFaces. These results suggest that the dynamic memory prototype enhances both image inversion and diversity in various settings.
However, the study also revealed a significant drop in \name's \textsc{Acc@1} from 94.06\% to 70.58\% when transitioning from FFHQ to MetFaces prior, indicating that the success of the attacks heavily depended on the diversity of prior information.

\begin{table}[!t]
    \centering
    \scalebox{0.85}{
    \begin{tabular}{ccc cccc c}
        \toprule
         ~ & \textbf{\textsc{Prior}}  & \textbf{\textsc{Method}} & \textbf{\textsc{$\uparrow$ Acc@1}} & \textbf{\textsc{$\downarrow l_{eval}^{2}$}} & \textbf{\textsc{$\downarrow l_{face}^{cos}$}} & \textbf{\textsc{$\downarrow$FID}} & \textbf{\textsc{$\uparrow$DIV}}  \\
         \midrule
        \multirow{4}{*}{\rotatebox{90}{CelebA}} & 
 	  \multirow{2}{*}{FFHQ} & PPA & 91.68 &  135.51 & 0.7303 & 58.16 & 0.3684 \\
 	 ~ & ~ & \name & \textbf{94.06} & \textbf{132.45} & \textbf{0.7030} & \textbf{55.92} &  \textbf{0.3723}
 \\
       \cmidrule{2-8}
     ~ &
 	 \multirow{2}{*}{Met.F} & PPA & 50.52 & 158.30 & 1.0708 & 89.74 & 0.1837 \\
 	 ~ & ~ & \name & \textbf{70.58} & \textbf{149.37} & \textbf{0.9821} & \textbf{89.33} & \textbf{0.1895} \\
 \midrule
 	 \multirow{ 4}{*}{\rotatebox{90}{FaceScrub}} &
 	 \multirow{2}{*}{FFHQ} & PPA & 75.80 & 120.70 & 0.7901 & 69.41 & 0.1419 \\
 	 ~  & ~ & \name & \textbf{86.98} & \textbf{120.00} & \textbf{0.7526} & \textbf{68.93} & \textbf{0.1466}
  \\
 	 \cmidrule{2-8}
 	 ~ & 
 	 \multirow{2}{*}{Met.F} & PPA &46.20 & 137.86 & 0.9524 & 94.78 & 0.1276\\
 	 ~ & ~ & \name &\textbf{68.70} & \textbf{131.92} & \textbf{0.8511} & \textbf{94.32} & \textbf{0.1579}\\
        \bottomrule
    \end{tabular}}
     \caption{Comparison results with various target datasets and image priors against ResNet-152 trained on CelebA and FaceScrub datasets. Met.F stands for MetFaces dataset.}
  \vspace{-0.8cm}
  \label{tbl:extend}
\end{table}
\vspace{-0.5pt}

\noindent \paragraph{\textbf{Benchmark results on non-privacy data.}}
We also provide the results on attacking the Stanford Dogs and MNIST dataset.

\begin{table}[t]
\renewcommand\arraystretch{1.3}
    \centering
    \small
    \scalebox{0.9}{
    \begin{tabular}{cccc cccc c}
        \toprule
         \textbf{\textsc{Architecture}}  & \textbf{\textsc{Method}} & \textbf{\textsc{$\uparrow$ Acc@1}} &  \textbf{\textsc{$\downarrow l_{eval}^{2}$}} & \textbf{\textsc{$\downarrow$FID}} & \textbf{\textsc{$\uparrow$Div}} \\
         \midrule
     \multirow{2}{*}{ResNeSt-101} & PPA & 97.60& 42.57 &  64.00& 0.3050 \\
  ~ &  \name &\textbf{99.25}  & \textbf{41.87}  & \textbf{63.43}& \textbf{0.3122}\\
  \multirow{2}{*}{ResNet-152} & PPA & 89.00& 40.75 & 63.32 & 0.2429 \\
 	 ~ &  \name & \textbf{97.50}& \textbf{39.71} & \textbf{61.46} & \textbf{0.2458}  \\
     \multirow{2}{*}{DenseNet-169} & PPA &  95.00 &38.45 &  63.17 & 0.2893 \\
 	  ~ & \name & \textbf{96.80} & \textbf{37.16} &  \textbf{61.73} & \textbf{0.3082}\\
        \bottomrule
    \end{tabular}}
     \caption{Comparison of attack results on Stanford Dogs dataset with AFHQ image prior. The \textsc{Div} is the mean of \textsc{Precision}, \textsc{Recall}, \textsc{Density} and \textsc{Coverage}, capturing the intra-class diversity of the generated samples. 
     }
  \vspace{-0.7cm}
    \label{tbl:ext_dog}
    \vspace{-0.3cm}
\end{table}

\noindent \textbf{Results on Stanford Dogs.}
The results in Table~\ref{tbl:ext_dog} show that \name surpasses PPA on three target models regarding attack success accuracy.
As confirmed by the FIDs, the average quality of our \name is higher than PPA. We visualize the generated images in Figure~\ref{fig:app_exp}a. It indicates the consistent performance of \name across different types of images.
Consequently, by representing target data from historical information, \name can capture more target-class sensitive information and expand the sample diversity.

\begin{table}[t]
\renewcommand\arraystretch{1.3}
    \centering
    \scalebox{0.9}{
    \begin{tabular}{cccccc}
    \toprule[1pt]
     \textbf{\textsc{Method}} & GMI \cite{zhang2020secret} & KED \cite{chen2021knowledge} & VMI \cite{wang2021variational}  & PPA \cite{struppek2022plug} & \name\\
     \midrule
     \textbf{\textsc{$\uparrow$ Acc@1}} & 27.53 & 38.61 & 28.60 & 45.60 & \textbf{48.72} \\
     \textbf{\textsc{$\downarrow$ FID}} & 90.64 &  72.41&  84.23& 58.05 & \textbf{57.43}\\
    \bottomrule
    \end{tabular}}
     \caption{Comparison results on attacking ResNet-18 with MNIST.}
    \label{tbl:ablation_1}
    \vspace{-0.8cm}
\end{table}

\noindent \textbf{Results on MNIST.}
To further show the effectiveness of \name and demonstrate its significance, we have run more experiments on MNIST \cite{lecun2010mnist}.
To ensure the StyleGAN2 contains hand-written image priors, we use all of the images with labels 5, 6, 7, 8, 9 as a public set while attacking the images with labels 0, 1, 2, 3, 4. All the images are resized to $32 \times 32$. The details of this experimental setting refer to the Appendix B.
In Table~\ref{tbl:ablation_3}, the ``Acc@1” column shows that the attack success rate promotion of \name on MNIST is 3.12 compared to state-of-the-art methods. Thus, \name is a more general method that works better for both natural and facial images and can be more practical.

We compare the reconstruction quality of different attacks in Figure~\ref{fig:app_exp}b.
The images produced by \name can capture most of the target class characteristics.
It can also be found that the generated images with successful attacks depend much on the image prior (\ie, images with labels [0,1,2,3,4]). The attack results on MNIST further demonstrate the superiority of \name.

\begin{figure}[t]
  \centering
  \includegraphics[width=1\linewidth]{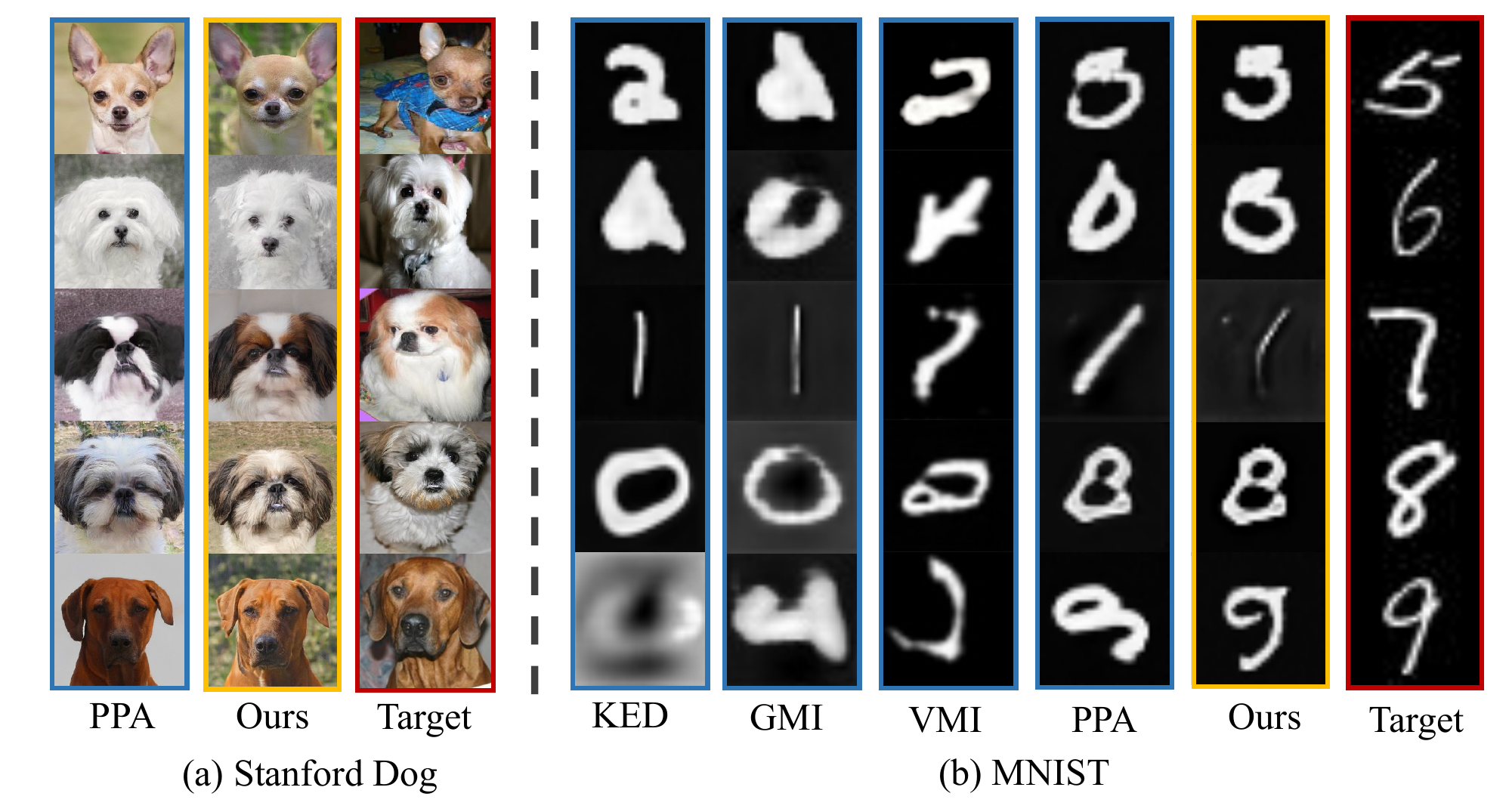}
  \vspace{-0.6cm}
  \caption{(a) Visual results against the ResNeSt-101 trained on Stanford Dog dataset. (b) Visual results against the ResNet-10 trained on MNIST dataset.}
  \vspace{-0.4cm}
  \label{fig:app_exp}
\end{figure}


\subsection{Ablation study}
\label{sec:ablation}
We conduct a series of experiments to examine the impact of several key aspects on the performance of the proposed \name attack. Specifically, we analyze the influence of different values of $\lambda_1$ and $\lambda_2$ (representing the impact of IMR and IDR on the model inversion attack), the prototype size $N_w$, the prototype size $\rho$, and the momentum coefficient $r$. Our evaluation was conducted using ResNet-18 trained on CelebA as the target model and StyleGAN2 with FFHQ as the image prior for the attack.

\noindent \paragraph{\textbf{Effect of $\lambda_1$ and $\lambda_2$.}}
$\lambda_1$ and $\lambda_2$ are the hyper-parameters to balance $\mathcal{L}_{\rm ce}$ and $\mathcal{L}_{\rm imr}$ with $\mathcal{L}_{\rm idr}$. Specifically, setting $\lambda_1=0$ or $\lambda_2=0$ allows us to evaluate the individual effect of prototype IDR and IMR.
As shown in Figure \ref{fig:imr_idr}, to better validate the two components of \name, studies are performed in a grid-search way.
In panels (a) and (b), we report the attack success rate and sample diversity over various $\lambda_1$ and $\lambda_2$. The optimal $\lambda_1$ and $\lambda_2$ are determined by $[Normalize(Acc@1) + Normalize(DIV)]/2$, shown in panel (c).

Notably, both IMR and IDR can benefit the attack performance.
Larger $\lambda_2$, \ie, \name with the higher weight of IDR can achieve better performance on attack accuracy, as shown in Figure \ref{fig:imr_idr}a. It indicates that the inter-class representation of IDR helps to capture more sensitive information of the targets.
Shown in Figure \ref{fig:imr_idr}b, the diversity performance of \name is mainly led by $\lambda_1$, \ie, the weight of IMR. It verifies the advantages of IMR that representing the target class distribution can encourage the $F_{\phi}$ to be more exploratory and augment sample diversity.
Finally, \name simultaneously considers the intra-relation and inter-relation of the privacy dataset, resulting in the best trade-off ($\lambda_1 = 0.3$ and $\lambda_2 = 0.7$) between sample quality and sample diversity. More results on the effect of IMR and IDR refer to the Appendix C.

\begin{figure}[t]
  \centering
  \includegraphics[width=1\linewidth]{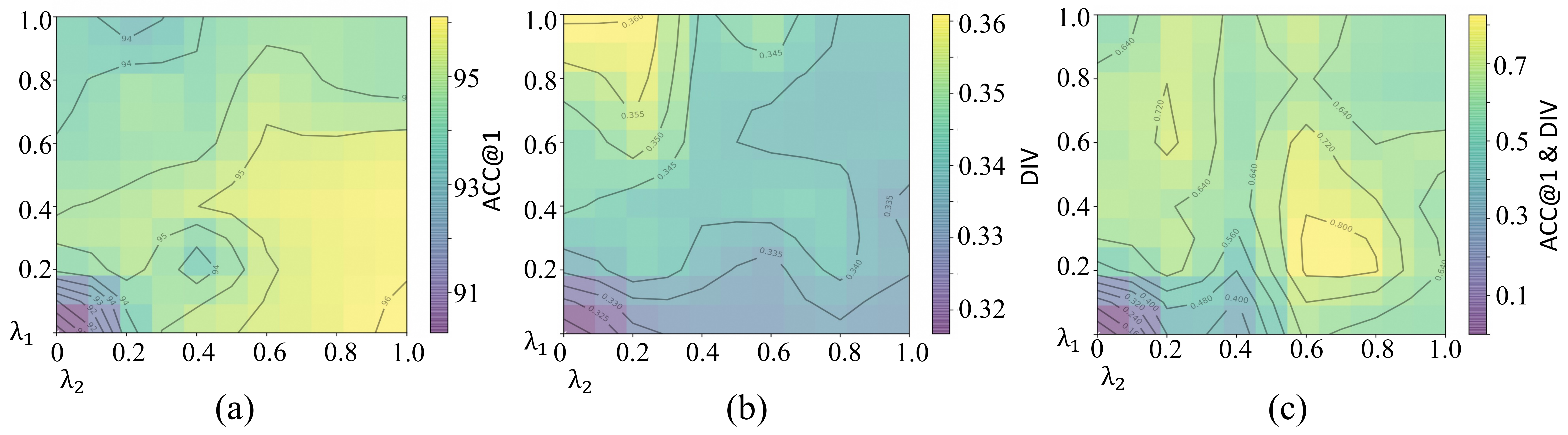}  
  \vspace{-0.8cm}
  \caption{Ablation study for the $\lambda_1$ and $\lambda_2$ on attack success rate and sample diversity, echoing (a): $\textsc{Acc@1}$ and (b): $\textsc{DIV}$. Chosen optimal hyper-parameters including both $\textsc{Acc@1}$ and $\textsc{DIV}$, echoing (c).
}
  \vspace{-0.6cm}
  \label{fig:imr_idr}
\end{figure}

\begin{table}[h]
\vspace{-0.2cm}
\renewcommand\arraystretch{1.3}
    \centering
    \scalebox{0.9}{
    \begin{tabular}{c ccccc}
        \toprule
        \textbf{\textsc{$N_w$}} &
        \textbf{\textsc{$\uparrow$ Acc@1}}& 
        \textbf{\textsc{$\downarrow l_{eval}^{2}$}} &
        \textbf{\textsc{$\downarrow l_{face}^{cos}$}} &\textbf{\textsc{$\downarrow$FID}} & \textbf{\textsc{$\uparrow$Div}} \\
        \midrule
     100  & \textbf{95.25} & \textbf{131.42} & \textbf{0.7449} & 90.17 & 0.3097 \\
 	 500 & 95.00 & 131.28 & 0.7448 & \textbf{89.03} & 0.3532 \\
     900 & 94.75 & 131.64 & 0.7471 & 89.97 & \textbf{0.3580}\\
        \bottomrule
    \end{tabular}}
     \caption{Results of \name, with different prototype size $N_w$, against ResNet-18 trained on CelebA.
     }
    \vspace{-0.6cm}
    \label{tbl:ablation_3}
    \vspace{-0.3cm}
\end{table}

\noindent \paragraph{\textbf{Effect of prototype size $N_w$.}}
The proposed intra-class multicentric representation introduces a learnable prototype set $W\in\mathbb{R}^{N_w \times N_d}$. We present the results of the study conducted to investigate the impact of prototype size $N_w$. $\rho$ is set to be half of $N_w$.
Echoing Table \ref{tbl:ablation_3}, \name with IMR can achieve more than 20\% DIV improvement as the prototype size increases from 100 to 900.
This suggests that IMR can effectively capture visual concepts of the target class through multi-centric representation. However, to balance performance with computation, we set the prototype size $N_w$ to 500 for the final experiments.

\begin{table}[ht]
\renewcommand\arraystretch{1.3}
    \centering
    \scalebox{0.9}{
    \begin{tabular}{c ccccc}
        \toprule
        \textbf{\textsc{$\rho$}} &
    \textbf{\textsc{$\uparrow$ Acc@1}}& 
        \textsc{$\downarrow l_{eval}^{2}$} &
        \textsc{$\downarrow l_{face}^{cos}$} &\textbf{\textsc{$\downarrow$FID}} & \textbf{\textsc{$\uparrow$Div}} \\
        \midrule
     100  &  95.00 & \textbf{131.16} & 0.7484 & 89.73 & 0.3361  \\
 	 250  & \textbf{95.00} & 131.28 & \textbf{0.7448} & \textbf{89.03} & 0.3532 \\  
 	 400 & 94.25 & 132.76 & 0.7563 & 89.16 & \textbf{0.3710} \\
        \bottomrule
    \end{tabular}}
     \caption{Ablation study of $\rho$ in IMR, which is the number of positive prototypes in $W$.
     }
  \vspace{-0.6cm}
    \label{tbl:ablation_rhp}
  \vspace{-0.4cm}
\end{table}

\noindent \paragraph{\textbf{Effect of prototype size $\rho$.}}
To evaluate the effect of prototype size of IMR on the attack performance, we vary the number of positive prototypes, represented by $\rho$, as shown in Table~\ref{tbl:ablation_rhp}. 
As $\rho$ rises from 100 to 400, \ie, more concepts for the target class are represented, which leads to an improvement in the diversity of the generated images. However, a larger value of $\rho$ may result in a trade-off between attack accuracy and sample diversity. To balance these two factors, we set $\rho=250$ as the optimal value for achieving better results.

\begin{table}[!h]
\renewcommand\arraystretch{1.3}
    \centering
    \scalebox{0.9}{
    \begin{tabular}{c ccccc}
        \toprule
        \textbf{\textsc{$r$}} &
        \textbf{\textsc{$\uparrow$ Acc@1}}& 
        \textsc{$\downarrow l_{eval}^{2}$} &
        \textsc{$\downarrow l_{face}^{cos}$} &\textbf{\textsc{$\downarrow$FID}} & \textbf{\textsc{$\uparrow$Div}} \\
        \midrule
     0.1 & 94.75 & 132.25 & 0.7482 & 89.08 & 0.3332\\
 	 0.3 & 95.00 & 131.96 & 0.7488 & 89.87 & 0.3477 \\
     0.5 & 95.00 & 132.53 & 0.7516 & 89.58 & \textbf{0.3662} \\
     0.7 & \textbf{95.75} & \textbf{131.48} & \textbf{0.7446} & 88.44 & 0.3413\\
     0.9 & 95.25 & 132.52 & 0.7482 & \textbf{88.35} & 0.3591 \\
        \bottomrule
    \end{tabular}}
     \caption{Ablation study of $r$. The performance is evaluated on attacking ResNet-18 with CelebA.
     }
     \vspace{-0.6cm}
    \label{tbl:ablation_4}
    \vspace{-0.2cm}
\end{table}
\noindent \paragraph{\textbf{Effect of $r$.}}
The momentum coefficient $r$ controls the updating the smoothness of the memory bank in IDR. 
As shown in Table \ref{tbl:ablation_4}, our method is generally not sensitive to the selection of $r$. We set $r=0.7$ for the optimal results. Noted that using too fast updating (\eg, 0.1) causes a small performance drop in attack accuracy. This is mainly because the updating of the memory bank is jointly conducted with the model training. 
If the memory bank is updated too quickly, it may not have enough time to properly learn from the prototypes being stored in the memory bank, which hinders the capacity of inter-class discriminability.

\subsection{Complexity Analysis}
Noted that incorporating \name into the training process incurs minimal additional computational cost compared to the baseline method, with no extra cost incurred during inference attacks. Specifically, the time complexity of IMR and IDR during training is $\mathcal{O}((K + N_w)N_dN_d)$. In terms of the attack on ResNet-18, it adds approximately 0.73M flops, which is negligible compared to the overall model.
Additionally, our method does not add any extra parameters to the generative model, i.e., it incurs no additional computational cost during the generation process.

\section{Conclusion and Future Work}
In this paper, we propose a Dynamic Memory Model Inversion Attack (\name), which improves the model inversion by reusing the memorized knowledge. Specifically, \name constructs an intra-class multicentric representation (IMR) term and an inter-class discriminative representation (IDR) term.
The learnable IMR models multiple concepts for representing the target class, which benefits from increasing sample diversity. Meanwhile, IDR injects the memorized knowledge into non-parametric prototypes to enforce the representations away from different categories, thus enhancing target class discriminability.
Overall, \name reports the state-of-the-art attack performance on multiple benchmarks. We will explore how to apply \name on black-box model inversion attacks in the future, where the attacker cannot obtain the information of the model and can only get the prediction result by querying the target model.




\bibliographystyle{ACM-Reference-Format}
\bibliography{sample-base}


\begin{thebibliography}{46}


\ifx \showCODEN    \undefined \def \showCODEN     #1{\unskip}     \fi
\ifx \showDOI      \undefined \def \showDOI       #1{#1}\fi
\ifx \showISBNx    \undefined \def \showISBNx     #1{\unskip}     \fi
\ifx \showISBNxiii \undefined \def \showISBNxiii  #1{\unskip}     \fi
\ifx \showISSN     \undefined \def \showISSN      #1{\unskip}     \fi
\ifx \showLCCN     \undefined \def \showLCCN      #1{\unskip}     \fi
\ifx \shownote     \undefined \def \shownote      #1{#1}          \fi
\ifx \showarticletitle \undefined \def \showarticletitle #1{#1}   \fi
\ifx \showURL      \undefined \def \showURL       {\relax}        \fi
\providecommand\bibfield[2]{#2}
\providecommand\bibinfo[2]{#2}
\providecommand\natexlab[1]{#1}
\providecommand\showeprint[2][]{arXiv:#2}

\bibitem[A{\"\i}vodji et~al\mbox{.}(2019)]%
        {aivodji2019gamin}
\bibfield{author}{\bibinfo{person}{Ulrich A{\"\i}vodji},
  \bibinfo{person}{S{\'e}bastien Gambs}, {and} \bibinfo{person}{Timon Ther}.}
  \bibinfo{year}{2019}\natexlab{}.
\newblock \showarticletitle{Gamin: An adversarial approach to black-box model
  inversion}.
\newblock \bibinfo{journal}{\emph{arXiv preprint arXiv:1909.11835}}
  (\bibinfo{year}{2019}).
\newblock


\bibitem[An et~al\mbox{.}(2022)]%
        {MIRROR}
\bibfield{author}{\bibinfo{person}{Shengwei An}, \bibinfo{person}{Guanhong
  Tao}, \bibinfo{person}{Qiuling Xu}, \bibinfo{person}{Yingqi Liu},
  \bibinfo{person}{Guangyu Shen}, \bibinfo{person}{Yuan Yao},
  \bibinfo{person}{Jingwei Xu}, {and} \bibinfo{person}{Xiangyu Zhang}.}
  \bibinfo{year}{2022}\natexlab{}.
\newblock \showarticletitle{MIRROR: Model Inversion for Deep Learning Network
  with High Fidelity}. In \bibinfo{booktitle}{\emph{Proceedings of the Network
  and Distributed Systems Security Symposium (NDSS 2022)}}.
\newblock


\bibitem[Chen et~al\mbox{.}(2020)]%
        {chen2020improved}
\bibfield{author}{\bibinfo{person}{Si Chen}, \bibinfo{person}{Ruoxi Jia}, {and}
  \bibinfo{person}{Guo-Jun Qi}.} \bibinfo{year}{2020}\natexlab{}.
\newblock \showarticletitle{Improved techniques for model inversion attacks}.
\newblock  (\bibinfo{year}{2020}).
\newblock


\bibitem[Chen et~al\mbox{.}(2021)]%
        {chen2021knowledge}
\bibfield{author}{\bibinfo{person}{Si Chen}, \bibinfo{person}{Mostafa Kahla},
  \bibinfo{person}{Ruoxi Jia}, {and} \bibinfo{person}{Guo-Jun Qi}.}
  \bibinfo{year}{2021}\natexlab{}.
\newblock \showarticletitle{Knowledge-Enriched Distributional Model Inversion
  Attacks}. In \bibinfo{booktitle}{\emph{Proceedings of the IEEE/CVF
  International Conference on Computer Vision}}. \bibinfo{pages}{16178--16187}.
\newblock


\bibitem[Chenshen et~al\mbox{.}(2018)]%
        {chenshen2018memory}
\bibfield{author}{\bibinfo{person}{WU Chenshen}, \bibinfo{person}{L HERRANZ},
  \bibinfo{person}{LIU Xialei}, {et~al\mbox{.}}}
  \bibinfo{year}{2018}\natexlab{}.
\newblock \showarticletitle{Memory replay GANs: Learning to generate images
  from new categories without forgetting [C]}. In \bibinfo{booktitle}{\emph{The
  32nd International Conference on Neural Information Processing Systems,
  Montr{\'e}al, Canada}}. \bibinfo{pages}{5966--5976}.
\newblock


\bibitem[Choi et~al\mbox{.}(2020)]%
        {choi2020stargan}
\bibfield{author}{\bibinfo{person}{Yunjey Choi}, \bibinfo{person}{Youngjung
  Uh}, \bibinfo{person}{Jaejun Yoo}, {and} \bibinfo{person}{Jung-Woo Ha}.}
  \bibinfo{year}{2020}\natexlab{}.
\newblock \showarticletitle{Stargan v2: Diverse image synthesis for multiple
  domains}. In \bibinfo{booktitle}{\emph{Proceedings of the IEEE/CVF conference
  on computer vision and pattern recognition}}. \bibinfo{pages}{8188--8197}.
\newblock


\bibitem[De~Boer et~al\mbox{.}(2005)]%
        {de2005tutorial}
\bibfield{author}{\bibinfo{person}{Pieter-Tjerk De~Boer},
  \bibinfo{person}{Dirk~P Kroese}, \bibinfo{person}{Shie Mannor}, {and}
  \bibinfo{person}{Reuven~Y Rubinstein}.} \bibinfo{year}{2005}\natexlab{}.
\newblock \showarticletitle{A tutorial on the cross-entropy method}.
\newblock \bibinfo{journal}{\emph{Annals of operations research}}
  \bibinfo{volume}{134}, \bibinfo{number}{1} (\bibinfo{year}{2005}),
  \bibinfo{pages}{19--67}.
\newblock


\bibitem[Deng et~al\mbox{.}(2009)]%
        {deng2009imagenet}
\bibfield{author}{\bibinfo{person}{Jia Deng}, \bibinfo{person}{Wei Dong},
  \bibinfo{person}{Richard Socher}, \bibinfo{person}{Li-Jia Li},
  \bibinfo{person}{Kai Li}, {and} \bibinfo{person}{Li Fei-Fei}.}
  \bibinfo{year}{2009}\natexlab{}.
\newblock \showarticletitle{Imagenet: A large-scale hierarchical image
  database}. In \bibinfo{booktitle}{\emph{2009 IEEE conference on computer
  vision and pattern recognition}}. Ieee, \bibinfo{pages}{248--255}.
\newblock


\bibitem[Fredrikson et~al\mbox{.}(2015)]%
        {fredrikson2015model}
\bibfield{author}{\bibinfo{person}{Matt Fredrikson}, \bibinfo{person}{Somesh
  Jha}, {and} \bibinfo{person}{Thomas Ristenpart}.}
  \bibinfo{year}{2015}\natexlab{}.
\newblock \showarticletitle{Model inversion attacks that exploit confidence
  information and basic countermeasures}. In
  \bibinfo{booktitle}{\emph{Proceedings of the 22nd ACM SIGSAC conference on
  computer and communications security}}. \bibinfo{pages}{1322--1333}.
\newblock


\bibitem[Fredrikson et~al\mbox{.}(2014)]%
        {fredrikson2014privacy}
\bibfield{author}{\bibinfo{person}{Matthew Fredrikson}, \bibinfo{person}{Eric
  Lantz}, \bibinfo{person}{Somesh Jha}, \bibinfo{person}{Simon Lin},
  \bibinfo{person}{David Page}, {and} \bibinfo{person}{Thomas Ristenpart}.}
  \bibinfo{year}{2014}\natexlab{}.
\newblock \showarticletitle{Privacy in pharmacogenetics: An $\{$End-to-End$\}$
  case study of personalized warfarin dosing}. In
  \bibinfo{booktitle}{\emph{23rd USENIX Security Symposium (USENIX Security
  14)}}. \bibinfo{pages}{17--32}.
\newblock


\bibitem[Guan et~al\mbox{.}(2019)]%
        {guan2019dcigan}
\bibfield{author}{\bibinfo{person}{Hongtao Guan}, \bibinfo{person}{Yijie Wang},
  \bibinfo{person}{Xingkong Ma}, {and} \bibinfo{person}{Yongmou Li}.}
  \bibinfo{year}{2019}\natexlab{}.
\newblock \showarticletitle{DCIGAN: a distributed class-incremental learning
  method based on generative adversarial networks}. In
  \bibinfo{booktitle}{\emph{2019 IEEE Intl Conf on Parallel \& Distributed
  Processing with Applications, Big Data \& Cloud Computing, Sustainable
  Computing \& Communications, Social Computing \& Networking
  (ISPA/BDCloud/SocialCom/SustainCom)}}. IEEE, \bibinfo{pages}{768--775}.
\newblock


\bibitem[He et~al\mbox{.}(2016)]%
        {he2016deep}
\bibfield{author}{\bibinfo{person}{Kaiming He}, \bibinfo{person}{Xiangyu
  Zhang}, \bibinfo{person}{Shaoqing Ren}, {and} \bibinfo{person}{Jian Sun}.}
  \bibinfo{year}{2016}\natexlab{}.
\newblock \showarticletitle{Deep residual learning for image recognition}. In
  \bibinfo{booktitle}{\emph{Proceedings of the IEEE conference on computer
  vision and pattern recognition}}. \bibinfo{pages}{770--778}.
\newblock


\bibitem[Heusel et~al\mbox{.}(2017)]%
        {heusel2017gans}
\bibfield{author}{\bibinfo{person}{Martin Heusel}, \bibinfo{person}{Hubert
  Ramsauer}, \bibinfo{person}{Thomas Unterthiner}, \bibinfo{person}{Bernhard
  Nessler}, {and} \bibinfo{person}{Sepp Hochreiter}.}
  \bibinfo{year}{2017}\natexlab{}.
\newblock \showarticletitle{Gans trained by a two time-scale update rule
  converge to a local nash equilibrium}.
\newblock \bibinfo{journal}{\emph{Advances in neural information processing
  systems}}  \bibinfo{volume}{30} (\bibinfo{year}{2017}).
\newblock


\bibitem[Hidano et~al\mbox{.}(2017)]%
        {hidano2017model}
\bibfield{author}{\bibinfo{person}{Seira Hidano}, \bibinfo{person}{Takao
  Murakami}, \bibinfo{person}{Shuichi Katsumata}, \bibinfo{person}{Shinsaku
  Kiyomoto}, {and} \bibinfo{person}{Goichiro Hanaoka}.}
  \bibinfo{year}{2017}\natexlab{}.
\newblock \showarticletitle{Model inversion attacks for prediction systems:
  Without knowledge of non-sensitive attributes}. In
  \bibinfo{booktitle}{\emph{2017 15th Annual Conference on Privacy, Security
  and Trust (PST)}}. IEEE, \bibinfo{pages}{115--11509}.
\newblock


\bibitem[Huang et~al\mbox{.}(2017)]%
        {huang2017densely}
\bibfield{author}{\bibinfo{person}{Gao Huang}, \bibinfo{person}{Zhuang Liu},
  \bibinfo{person}{Laurens Van Der~Maaten}, {and} \bibinfo{person}{Kilian~Q
  Weinberger}.} \bibinfo{year}{2017}\natexlab{}.
\newblock \showarticletitle{Densely connected convolutional networks}. In
  \bibinfo{booktitle}{\emph{Proceedings of the IEEE conference on computer
  vision and pattern recognition}}. \bibinfo{pages}{4700--4708}.
\newblock


\bibitem[Karras et~al\mbox{.}(2020a)]%
        {karras2020training}
\bibfield{author}{\bibinfo{person}{Tero Karras}, \bibinfo{person}{Miika
  Aittala}, \bibinfo{person}{Janne Hellsten}, \bibinfo{person}{Samuli Laine},
  \bibinfo{person}{Jaakko Lehtinen}, {and} \bibinfo{person}{Timo Aila}.}
  \bibinfo{year}{2020}\natexlab{a}.
\newblock \showarticletitle{Training generative adversarial networks with
  limited data}.
\newblock \bibinfo{journal}{\emph{Advances in Neural Information Processing
  Systems}}  \bibinfo{volume}{33} (\bibinfo{year}{2020}),
  \bibinfo{pages}{12104--12114}.
\newblock


\bibitem[Karras et~al\mbox{.}(2019)]%
        {karras2019style}
\bibfield{author}{\bibinfo{person}{Tero Karras}, \bibinfo{person}{Samuli
  Laine}, {and} \bibinfo{person}{Timo Aila}.} \bibinfo{year}{2019}\natexlab{}.
\newblock \showarticletitle{A style-based generator architecture for generative
  adversarial networks}. In \bibinfo{booktitle}{\emph{Proceedings of the
  IEEE/CVF conference on computer vision and pattern recognition}}.
  \bibinfo{pages}{4401--4410}.
\newblock


\bibitem[Karras et~al\mbox{.}(2020b)]%
        {karras2020analyzing}
\bibfield{author}{\bibinfo{person}{Tero Karras}, \bibinfo{person}{Samuli
  Laine}, \bibinfo{person}{Miika Aittala}, \bibinfo{person}{Janne Hellsten},
  \bibinfo{person}{Jaakko Lehtinen}, {and} \bibinfo{person}{Timo Aila}.}
  \bibinfo{year}{2020}\natexlab{b}.
\newblock \showarticletitle{Analyzing and improving the image quality of
  stylegan}. In \bibinfo{booktitle}{\emph{Proceedings of the IEEE/CVF
  conference on computer vision and pattern recognition}}.
  \bibinfo{pages}{8110--8119}.
\newblock


\bibitem[Khosravy et~al\mbox{.}(2022)]%
        {khosravy2022model}
\bibfield{author}{\bibinfo{person}{Mahdi Khosravy}, \bibinfo{person}{Kazuaki
  Nakamura}, \bibinfo{person}{Yuki Hirose}, \bibinfo{person}{Naoko Nitta},
  {and} \bibinfo{person}{Noboru Babaguchi}.} \bibinfo{year}{2022}\natexlab{}.
\newblock \showarticletitle{Model Inversion Attack by Integration of Deep
  Generative Models: Privacy-Sensitive Face Generation from a Face Recognition
  System}.
\newblock \bibinfo{journal}{\emph{IEEE Transactions on Information Forensics
  and Security}}  \bibinfo{volume}{17} (\bibinfo{year}{2022}),
  \bibinfo{pages}{357--372}.
\newblock


\bibitem[Kingma and Ba(2015)]%
        {kingma2015adam}
\bibfield{author}{\bibinfo{person}{Diederik~P Kingma} {and}
  \bibinfo{person}{Jimmy Ba}.} \bibinfo{year}{2015}\natexlab{}.
\newblock \showarticletitle{Adam: A Method for Stochastic Optimization}. In
  \bibinfo{booktitle}{\emph{ICLR (Poster)}}.
\newblock


\bibitem[Kirkpatrick et~al\mbox{.}(2017)]%
        {kirkpatrick2017overcoming}
\bibfield{author}{\bibinfo{person}{James Kirkpatrick}, \bibinfo{person}{Razvan
  Pascanu}, \bibinfo{person}{Neil Rabinowitz}, \bibinfo{person}{Joel Veness},
  \bibinfo{person}{Guillaume Desjardins}, \bibinfo{person}{Andrei~A Rusu},
  \bibinfo{person}{Kieran Milan}, \bibinfo{person}{John Quan},
  \bibinfo{person}{Tiago Ramalho}, \bibinfo{person}{Agnieszka
  Grabska-Barwinska}, {et~al\mbox{.}}} \bibinfo{year}{2017}\natexlab{}.
\newblock \showarticletitle{Overcoming catastrophic forgetting in neural
  networks}.
\newblock \bibinfo{journal}{\emph{Proceedings of the national academy of
  sciences}} \bibinfo{volume}{114}, \bibinfo{number}{13}
  (\bibinfo{year}{2017}), \bibinfo{pages}{3521--3526}.
\newblock


\bibitem[Kynk{\"a}{\"a}nniemi et~al\mbox{.}(2019)]%
        {kynkaanniemi2019improved}
\bibfield{author}{\bibinfo{person}{Tuomas Kynk{\"a}{\"a}nniemi},
  \bibinfo{person}{Tero Karras}, \bibinfo{person}{Samuli Laine},
  \bibinfo{person}{Jaakko Lehtinen}, {and} \bibinfo{person}{Timo Aila}.}
  \bibinfo{year}{2019}\natexlab{}.
\newblock \showarticletitle{Improved precision and recall metric for assessing
  generative models}.
\newblock \bibinfo{journal}{\emph{Advances in Neural Information Processing
  Systems}}  \bibinfo{volume}{32} (\bibinfo{year}{2019}).
\newblock


\bibitem[LeCun et~al\mbox{.}(2010)]%
        {lecun2010mnist}
\bibfield{author}{\bibinfo{person}{Yann LeCun}, \bibinfo{person}{Corinna
  Cortes}, {and} \bibinfo{person}{Chris Burges}.}
  \bibinfo{year}{2010}\natexlab{}.
\newblock \bibinfo{title}{MNIST handwritten digit database}.
\newblock
\newblock


\bibitem[Lesort et~al\mbox{.}(2019)]%
        {lesort2019generative}
\bibfield{author}{\bibinfo{person}{Timoth{\'e}e Lesort}, \bibinfo{person}{Hugo
  Caselles-Dupr{\'e}}, \bibinfo{person}{Michael Garcia-Ortiz},
  \bibinfo{person}{Andrei Stoian}, {and} \bibinfo{person}{David Filliat}.}
  \bibinfo{year}{2019}\natexlab{}.
\newblock \showarticletitle{Generative models from the perspective of continual
  learning}. In \bibinfo{booktitle}{\emph{2019 International Joint Conference
  on Neural Networks (IJCNN)}}. IEEE, \bibinfo{pages}{1--8}.
\newblock


\bibitem[Liu et~al\mbox{.}(2015)]%
        {liu2015deep}
\bibfield{author}{\bibinfo{person}{Ziwei Liu}, \bibinfo{person}{Ping Luo},
  \bibinfo{person}{Xiaogang Wang}, {and} \bibinfo{person}{Xiaoou Tang}.}
  \bibinfo{year}{2015}\natexlab{}.
\newblock \showarticletitle{Deep learning face attributes in the wild}. In
  \bibinfo{booktitle}{\emph{Proceedings of the IEEE international conference on
  computer vision}}. \bibinfo{pages}{3730--3738}.
\newblock


\bibitem[Naeem et~al\mbox{.}(2020)]%
        {naeem2020reliable}
\bibfield{author}{\bibinfo{person}{Muhammad~Ferjad Naeem},
  \bibinfo{person}{Seong~Joon Oh}, \bibinfo{person}{Youngjung Uh},
  \bibinfo{person}{Yunjey Choi}, {and} \bibinfo{person}{Jaejun Yoo}.}
  \bibinfo{year}{2020}\natexlab{}.
\newblock \showarticletitle{Reliable fidelity and diversity metrics for
  generative models}. In \bibinfo{booktitle}{\emph{International Conference on
  Machine Learning}}. PMLR, \bibinfo{pages}{7176--7185}.
\newblock


\bibitem[Ng and Winkler(2014)]%
        {ng2014data}
\bibfield{author}{\bibinfo{person}{Hong-Wei Ng} {and} \bibinfo{person}{Stefan
  Winkler}.} \bibinfo{year}{2014}\natexlab{}.
\newblock \showarticletitle{A data-driven approach to cleaning large face
  datasets}. In \bibinfo{booktitle}{\emph{2014 IEEE international conference on
  image processing (ICIP)}}. IEEE, \bibinfo{pages}{343--347}.
\newblock


\bibitem[Parkhi et~al\mbox{.}(2015)]%
        {parkhi2015deep}
\bibfield{author}{\bibinfo{person}{Omkar~M Parkhi}, \bibinfo{person}{Andrea
  Vedaldi}, {and} \bibinfo{person}{Andrew Zisserman}.}
  \bibinfo{year}{2015}\natexlab{}.
\newblock \showarticletitle{Deep face recognition}.
\newblock  (\bibinfo{year}{2015}).
\newblock


\bibitem[Radford et~al\mbox{.}(2021)]%
        {radford2021learning}
\bibfield{author}{\bibinfo{person}{Alec Radford}, \bibinfo{person}{Jong~Wook
  Kim}, \bibinfo{person}{Chris Hallacy}, \bibinfo{person}{Aditya Ramesh},
  \bibinfo{person}{Gabriel Goh}, \bibinfo{person}{Sandhini Agarwal},
  \bibinfo{person}{Girish Sastry}, \bibinfo{person}{Amanda Askell},
  \bibinfo{person}{Pamela Mishkin}, \bibinfo{person}{Jack Clark},
  {et~al\mbox{.}}} \bibinfo{year}{2021}\natexlab{}.
\newblock \showarticletitle{Learning transferable visual models from natural
  language supervision}. In \bibinfo{booktitle}{\emph{International conference
  on machine learning}}. PMLR, \bibinfo{pages}{8748--8763}.
\newblock


\bibitem[Schroff et~al\mbox{.}(2015)]%
        {schroff2015facenet}
\bibfield{author}{\bibinfo{person}{Florian Schroff}, \bibinfo{person}{Dmitry
  Kalenichenko}, {and} \bibinfo{person}{James Philbin}.}
  \bibinfo{year}{2015}\natexlab{}.
\newblock \showarticletitle{Facenet: A unified embedding for face recognition
  and clustering}. In \bibinfo{booktitle}{\emph{Proceedings of the IEEE
  conference on computer vision and pattern recognition}}.
  \bibinfo{pages}{815--823}.
\newblock


\bibitem[Sciences(2015)]%
        {academy2015stratified}
\bibfield{author}{\bibinfo{person}{Academy~Medical Sciences}.}
  \bibinfo{year}{2015}\natexlab{}.
\newblock \showarticletitle{Stratified, personalised or P4 medicine: a new
  direction for placing the patient at the centre of healthcare and health
  education}.
\newblock  (\bibinfo{year}{2015}).
\newblock


\bibitem[Seff et~al\mbox{.}(2017)]%
        {seff2017continual}
\bibfield{author}{\bibinfo{person}{Ari Seff}, \bibinfo{person}{Alex Beatson},
  \bibinfo{person}{Daniel Suo}, {and} \bibinfo{person}{Han Liu}.}
  \bibinfo{year}{2017}\natexlab{}.
\newblock \showarticletitle{Continual learning in generative adversarial nets}.
\newblock \bibinfo{journal}{\emph{arXiv preprint arXiv:1705.08395}}
  (\bibinfo{year}{2017}).
\newblock


\bibitem[Shen et~al\mbox{.}(2019)]%
        {shen2019defending}
\bibfield{author}{\bibinfo{person}{Chaomin Shen}, \bibinfo{person}{Yaxin Peng},
  \bibinfo{person}{Guixu Zhang}, {and} \bibinfo{person}{Jinsong Fan}.}
  \bibinfo{year}{2019}\natexlab{}.
\newblock \showarticletitle{Defending against adversarial attacks by
  suppressing the largest eigenvalue of fisher information matrix}.
\newblock \bibinfo{journal}{\emph{arXiv preprint arXiv:1909.06137}}
  (\bibinfo{year}{2019}).
\newblock


\bibitem[Struppek et~al\mbox{.}(2022)]%
        {struppek2022plug}
\bibfield{author}{\bibinfo{person}{Lukas Struppek}, \bibinfo{person}{Dominik
  Hintersdorf}, \bibinfo{person}{Antonio De~Almeida Correia},
  \bibinfo{person}{Antonia Adler}, {and} \bibinfo{person}{Kristian Kersting}.}
  \bibinfo{year}{2022}\natexlab{}.
\newblock \showarticletitle{Plug \& Play Attacks: Towards Robust and Flexible
  Model Inversion Attacks}.
\newblock \bibinfo{journal}{\emph{arXiv preprint arXiv:2201.12179}}
  (\bibinfo{year}{2022}).
\newblock


\bibitem[Szegedy et~al\mbox{.}(2016)]%
        {szegedy2016rethinking}
\bibfield{author}{\bibinfo{person}{Christian Szegedy}, \bibinfo{person}{Vincent
  Vanhoucke}, \bibinfo{person}{Sergey Ioffe}, \bibinfo{person}{Jon Shlens},
  {and} \bibinfo{person}{Zbigniew Wojna}.} \bibinfo{year}{2016}\natexlab{}.
\newblock \showarticletitle{Rethinking the inception architecture for computer
  vision}. In \bibinfo{booktitle}{\emph{Proceedings of the IEEE conference on
  computer vision and pattern recognition}}. \bibinfo{pages}{2818--2826}.
\newblock


\bibitem[Thanh-Tung and Tran(2018)]%
        {thanh2018catastrophic}
\bibfield{author}{\bibinfo{person}{Hoang Thanh-Tung} {and}
  \bibinfo{person}{Truyen Tran}.} \bibinfo{year}{2018}\natexlab{}.
\newblock \showarticletitle{On catastrophic forgetting in generative
  adversarial networks}.
\newblock \bibinfo{journal}{\emph{arXiv preprint arXiv:1807.04015}}
  (\bibinfo{year}{2018}).
\newblock


\bibitem[Wang et~al\mbox{.}(2021)]%
        {wang2021variational}
\bibfield{author}{\bibinfo{person}{Kuan-Chieh Wang}, \bibinfo{person}{Yan Fu},
  \bibinfo{person}{Ke Li}, \bibinfo{person}{Ashish Khisti},
  \bibinfo{person}{Richard Zemel}, {and} \bibinfo{person}{Alireza Makhzani}.}
  \bibinfo{year}{2021}\natexlab{}.
\newblock \showarticletitle{Variational Model Inversion Attacks}.
\newblock \bibinfo{journal}{\emph{Advances in Neural Information Processing
  Systems}}  \bibinfo{volume}{34} (\bibinfo{year}{2021}).
\newblock


\bibitem[Wu and Yan(2017)]%
        {wu2017session}
\bibfield{author}{\bibinfo{person}{Chen Wu} {and} \bibinfo{person}{Ming Yan}.}
  \bibinfo{year}{2017}\natexlab{}.
\newblock \showarticletitle{Session-aware information embedding for e-commerce
  product recommendation}. In \bibinfo{booktitle}{\emph{Proceedings of the 2017
  ACM on conference on information and knowledge management}}.
  \bibinfo{pages}{2379--2382}.
\newblock


\bibitem[Wu et~al\mbox{.}(2016)]%
        {wu2016methodology}
\bibfield{author}{\bibinfo{person}{Xi Wu}, \bibinfo{person}{Matthew
  Fredrikson}, \bibinfo{person}{Somesh Jha}, {and} \bibinfo{person}{Jeffrey~F
  Naughton}.} \bibinfo{year}{2016}\natexlab{}.
\newblock \showarticletitle{A methodology for formalizing model-inversion
  attacks}. In \bibinfo{booktitle}{\emph{2016 IEEE 29th Computer Security
  Foundations Symposium (CSF)}}. IEEE, \bibinfo{pages}{355--370}.
\newblock


\bibitem[Yang et~al\mbox{.}(2019)]%
        {yang2019adversarial}
\bibfield{author}{\bibinfo{person}{Ziqi Yang}, \bibinfo{person}{Ee-Chien
  Chang}, {and} \bibinfo{person}{Zhenkai Liang}.}
  \bibinfo{year}{2019}\natexlab{}.
\newblock \showarticletitle{Adversarial neural network inversion via auxiliary
  knowledge alignment}.
\newblock \bibinfo{journal}{\emph{arXiv preprint arXiv:1902.08552}}
  (\bibinfo{year}{2019}).
\newblock


\bibitem[Zhai et~al\mbox{.}(2020)]%
        {zhai2020piggyback}
\bibfield{author}{\bibinfo{person}{Mengyao Zhai}, \bibinfo{person}{Lei Chen},
  \bibinfo{person}{Jiawei He}, \bibinfo{person}{Megha Nawhal},
  \bibinfo{person}{Frederick Tung}, {and} \bibinfo{person}{Greg Mori}.}
  \bibinfo{year}{2020}\natexlab{}.
\newblock \showarticletitle{Piggyback gan: Efficient lifelong learning for
  image conditioned generation}. In \bibinfo{booktitle}{\emph{Computer
  Vision--ECCV 2020: 16th European Conference, Glasgow, UK, August 23--28,
  2020, Proceedings, Part XXI 16}}. Springer, \bibinfo{pages}{397--413}.
\newblock


\bibitem[Zhai et~al\mbox{.}(2019)]%
        {zhai2019lifelong}
\bibfield{author}{\bibinfo{person}{Mengyao Zhai}, \bibinfo{person}{Lei Chen},
  \bibinfo{person}{Frederick Tung}, \bibinfo{person}{Jiawei He},
  \bibinfo{person}{Megha Nawhal}, {and} \bibinfo{person}{Greg Mori}.}
  \bibinfo{year}{2019}\natexlab{}.
\newblock \showarticletitle{Lifelong gan: Continual learning for conditional
  image generation}. In \bibinfo{booktitle}{\emph{Proceedings of the IEEE/CVF
  international conference on computer vision}}. \bibinfo{pages}{2759--2768}.
\newblock


\bibitem[Zhang et~al\mbox{.}(2020b)]%
        {zhang2020resnest}
\bibfield{author}{\bibinfo{person}{Hang Zhang}, \bibinfo{person}{Chongruo Wu},
  \bibinfo{person}{Zhongyue Zhang}, \bibinfo{person}{Yi Zhu},
  \bibinfo{person}{Haibin Lin}, \bibinfo{person}{Zhi Zhang},
  \bibinfo{person}{Yue Sun}, \bibinfo{person}{Tong He}, \bibinfo{person}{Jonas
  Mueller}, \bibinfo{person}{R Manmatha}, {et~al\mbox{.}}}
  \bibinfo{year}{2020}\natexlab{b}.
\newblock \showarticletitle{Resnest: Split-attention networks}.
\newblock \bibinfo{journal}{\emph{arXiv preprint arXiv:2004.08955}}
  (\bibinfo{year}{2020}).
\newblock


\bibitem[Zhang et~al\mbox{.}(2020a)]%
        {zhang2020secret}
\bibfield{author}{\bibinfo{person}{Yuheng Zhang}, \bibinfo{person}{Ruoxi Jia},
  \bibinfo{person}{Hengzhi Pei}, \bibinfo{person}{Wenxiao Wang},
  \bibinfo{person}{Bo Li}, {and} \bibinfo{person}{Dawn Song}.}
  \bibinfo{year}{2020}\natexlab{a}.
\newblock \showarticletitle{The secret revealer: Generative model-inversion
  attacks against deep neural networks}. In
  \bibinfo{booktitle}{\emph{Proceedings of the IEEE/CVF Conference on Computer
  Vision and Pattern Recognition}}. \bibinfo{pages}{253--261}.
\newblock


\bibitem[Zhao et~al\mbox{.}(2019)]%
        {zhao2019adversarial}
\bibfield{author}{\bibinfo{person}{Chenxiao Zhao}, \bibinfo{person}{P~Thomas
  Fletcher}, \bibinfo{person}{Mixue Yu}, \bibinfo{person}{Yaxin Peng},
  \bibinfo{person}{Guixu Zhang}, {and} \bibinfo{person}{Chaomin Shen}.}
  \bibinfo{year}{2019}\natexlab{}.
\newblock \showarticletitle{The adversarial attack and detection under the
  fisher information metric}. In \bibinfo{booktitle}{\emph{Proceedings of the
  AAAI Conference on Artificial Intelligence}}, Vol.~\bibinfo{volume}{33}.
  \bibinfo{pages}{5869--5876}.
\newblock


\bibitem[Zhao et~al\mbox{.}(2021)]%
        {zhao2021exploiting}
\bibfield{author}{\bibinfo{person}{Xuejun Zhao}, \bibinfo{person}{Wencan
  Zhang}, \bibinfo{person}{Xiaokui Xiao}, {and} \bibinfo{person}{Brian Lim}.}
  \bibinfo{year}{2021}\natexlab{}.
\newblock \showarticletitle{Exploiting explanations for model inversion
  attacks}. In \bibinfo{booktitle}{\emph{Proceedings of the IEEE/CVF
  International Conference on Computer Vision}}. \bibinfo{pages}{682--692}.
\newblock


\end{thebibliography}

\end{document}